\documentclass[sigconf]{aamas} 

\usepackage{balance} 

\usepackage{savesym}
\savesymbol{Bbbk}
\savesymbol{openbox}

\usepackage{xcolor}
\usepackage{url}
\usepackage[utf8]{inputenc}
\usepackage{graphicx}
\usepackage{amsmath}
\usepackage{amssymb}
\usepackage{amsthm}
\usepackage{MnSymbol}
\usepackage{booktabs}
\usepackage{algorithm}
\usepackage{multirow}

\restoresymbol{NEW}{Bbbk}
\restoresymbol{NEW}{openbox}
\usepackage{tcolorbox}

\makeatletter
\gdef\@copyrightpermission{
  \begin{minipage}{0.2\columnwidth}
   \href{https://creativecommons.org/licenses/by/4.0/}{\includegraphics[width=0.90\textwidth]{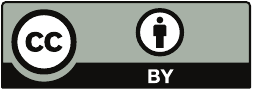}}
  \end{minipage}\hfill
  \begin{minipage}{0.8\columnwidth}
   \href{https://creativecommons.org/licenses/by/4.0/}{This work is licensed under a Creative Commons Attribution International 4.0 License.}
  \end{minipage}
  \vspace{5pt}
}
\makeatother

\setcopyright{ifaamas}
\acmConference[AAMAS '25]{Proc.\@ of the 24th International Conference
on Autonomous Agents and Multiagent Systems (AAMAS 2025)}{May 19 -- 23, 2025}
{Detroit, Michigan, USA}{Y.~Vorobeychik, S.~Das, A.~Nowé  (eds.)}
\copyrightyear{2025}
\acmYear{2025}
\acmDOI{}
\acmPrice{}
\acmISBN{}

\acmSubmissionID{765}

\title{Free Argumentative Exchanges for Explaining Image Classifiers}

\author{Avinash Kori}
\affiliation{
  \institution{Imperial College London}
  \country{United Kingdom}}
\email{a.kori21@imperial.ac.uk}

\author{Antonio Rago}
\affiliation{
  \institution{Imperial College London}
  \country{United Kingdom}}
\email{a.rago@imperial.ac.uk}

\author{Francesca Toni}
\affiliation{
  \institution{Imperial College London}
  \country{United Kingdom}}
\orcid{0000-0001-8194-1459}
\email{f.toni@imperial.ac.uk}

\begin{abstract}

Deep learning models are powerful image classifiers but their opacity hinders their trustworthiness.
Explanation methods for capturing the reasoning process within these classifiers faithfully and in a 
\CR{clear manner} are scarce, due to their sheer complexity and size.
\CR{We} provide a solution for this problem by defining a novel method for explaining the outputs of image classifiers with debates  between two agents, each arguing for a particular class. 
We obtain these debates as concrete instances of \emph{Free Argumentative eXchanges} (FAXs), a novel  argumentation-based multi-agent framework
allowing agents to internalise opinions by other agents differently than originally stated. 
We 
define two metrics \CR{(\emph{consensus} and \emph{persuasion rate})} to assess 
the usefulness of FAXs as argumentative explanations for image classifiers.
We then conduct a number of
empirical experiments 
showing that FAXs 
perform well along these metrics as well as being more faithful to the image  classifiers than conventional, non-argumentative explanation methods.
\CR{All our implementations can be found at \url{https://github.com/koriavinash1/FAX}}.

\end{abstract}


\usepackage[noend]{algpseudocode}
\usepackage{amssymb}
\usepackage{xcolor}

\newcommand{\todo}[1]{\textcolor{magenta}{$*$ #1}}
\newcommand{\todoE}[1]{
}
\newcommand{\delete}[1]{}

\newcommand{\CR}[1]{\textcolor{black}{#1}}

\def\FAXIC{FAX
}

\DeclareMathOperator*{\argmin}{argmin}

\newtheorem{example}{Example}

\newtheorem{definition}{Definition}

\newtheorem{proposition}{Proposition}

\newcommand{\Args}{\ensuremath{\mathbf{X}}}
\newcommand{\Atts}{\ensuremath{\mathbf{A}}}
\newcommand{\Supps}{\ensuremath{\mathbf{S}}}
\newcommand{\BAF}{\ensuremath{\mathbf{B}}}
\newcommand{\QBAF}{\ensuremath{\mathbf{Q}}}
\newcommand{\argument}{\ensuremath{\alpha}}
\newcommand{\otherargument}{\ensuremath{\beta}}

\newcommand{\SF}{\ensuremath{\sigma}}
\newcommand{\Range}{\ensuremath{\mathbb{I}}}
\newcommand{\RPos}{\ensuremath{\Range_+}}
\newcommand{\RNeg}{\ensuremath{\Range_-}}

\newcommand{\indexAg}{\ensuremath{
i}} 

\newcommand{\qAgentI}{\ensuremath{
q^1}}
\newcommand{\qAgentJ}{\ensuremath{
q^2}}
\newcommand{\qAgentW}{\ensuremath{
q^{\indexAg}}}

\newcommand{\RangeW}{\ensuremath{
\Range^{\indexAg}}}
\newcommand{\BAFW}{\ensuremath{
\BAF^{\indexAg}}}
\newcommand{\SFW}{\ensuremath{
\SF^{\indexAg}}}
\newcommand{\ArgsW}{\ensuremath{\Args^{\indexAg}}}
\newcommand{\AttsW}{\ensuremath{\Atts^\indexAg}}
\newcommand{\SuppsW}{\ensuremath{\Supps^\indexAg}}

\newcommand{\AgentI}{\ensuremath{
\Agents^i}}

\newcommand{\Stance}{\ensuremath{\Sigma}}
\newcommand{\Exch}{\ensuremath{x}}
\newcommand{\argpaths}{\mathsf{paths}}

\newcommand{\SpeakerM}{\ensuremath{\mathcal{M}}}

\newcommand{\Agents}{\ensuremath{\mathcal{
A}}}

\newcommand{\AGi}{\ensuremath{i}}

\newcommand{\BAFi}{\ensuremath{\BAF^\AGi}}

\newcommand{\ArgsI}{\ensuremath{\Args^\AGi}}

\newcommand{\AttsI}{\ensuremath{\Atts^\AGi}}

\newcommand{\SuppsI}{\ensuremath{\Supps^\AGi}}

\newcommand{\SFi}{\ensuremath{\SF^\AGi}}

\newcommand{\RangeI}{\ensuremath{\Range^\AGi}}

\newcommand{\RPosI}{\ensuremath{\RPos^\AGi}}
\newcommand{\RNegI}{\ensuremath{\RNeg^\AGi}}

\newcommand{\StanceI}{\ensuremath{\Stance^\AGi}}

\newcommand{\reward}{\ensuremath{r}}
\newcommand{\rewardI}{\ensuremath{\reward^\AGi}}

\newcommand{\BAFx}{\ensuremath{\BAF^\Exch}}

\newcommand{\ArgsX}{\ensuremath{\Args^\Exch}}
\newcommand{\AttsX}{\ensuremath{\Atts^\Exch}}
\newcommand{\SuppsX}{\ensuremath{\Supps^\Exch}}

\newcommand{\arge}{\ensuremath{e}}
\newcommand{\arga}{\ensuremath{a}}
\newcommand{\argb}{\ensuremath{b}}
\newcommand{\argc}{\ensuremath{c}}

\newcommand{\argx}{\ensuremath{x}}
\newcommand{\argy}{\ensuremath{y}}

\newcommand{\BibTeX}{\rm B\kern-.05em{\sc i\kern-.025em b}\kern-.08em\TeX}

\begin{document}


\pagestyle{fancy}
\fancyhead{}


\maketitle 



\section{Introduction}

With the increasing complexity and widespread deployment of deep learning models in our daily lives, the interpretation and explanation of these models' decisions have become a central focus in recent eXplainable Artificial Intelligence (XAI) literature \cite{lime,SHAP,deeplift}. 
Many existing approaches for explaining image classification models heavily rely on heatmaps and segments to localize regions of interest in input images that contribute to the model's output \cite{selvaraju2017grad,gradcampp,integratedcam}, 
typically offering static input-output-based explanations while lacking deeper insights into the underlying model being explained.
The literature has repeatedly highlighted the need for deeper and more dynamic explanations \cite{miller2019explanation,lakkaraju2022rethinking,madumal2019grounded},  
highlighting the input-output relationships of the model but also delving into its internal mechanisms and elucidating 
the model's reasoning. Also, there is an ongoing debate about the necessity of interactive 
explanations \cite{cawsey1991generating,miller2019explanation} and explanations that are contrastive and selected \cite{miller2019explanation}.

Dialogue-based explanations have been advocated as being useful in understanding the inner working of deep learning models \cite{lakkaraju2022rethinking}
.
\cite{wang2019deliberative} argues that explanations are especially important when the model has high uncertainty about the output when the prediction oscillates between different classes resulting in different interpretations of the model behaviour.
\cite{Rago_23} proposes 
argumentative exchanges to explain models via interactions amongst agents.
Motivated by these varied lines of work, the main focus of this work is to extract explanations for image classifiers as 
debates between two artificial agents, arguing, in the spirit of bipolar argumentation~\cite{Cayrol_05}, for and against 
the classifiers' outputs for given inputs.

\begin{example}
\label{ex:illu} 
As an abstract illustration, 
    consider two agents $\Agents^1$ and $\Agents^2$ as outlined in Figure \ref{fig:example} (black and grey, respectively). 
    Each agent has an initial  perception (at timestep $t=0$) of their environment as a bipolar argumentation framework (BAF) about a topic of interest ($a$ in the figure) which we consider as private. 
    As agents start debating, they share  their knowledge
    (see the exchange BAF figure \ref{fig:example}) and may expand their private BAFs (e.g. $\Agents^2$ learns that $b$ supports $a$ at timestep $t=1$ and agent $\Agents^1$ learns that $c$ attacks $a$ at timestep $t=2$). As in human debates,
    as agents expand their knowledge they may see things differently from the other agents (e.g. at the alternative timestep $t=2'$ $\Agents^1$ sees $\Agents^2s$'s attack  from $c$ to $a$ as a support). 
\end{example}

Overall,  we make the following contributions:

\begin{itemize}
    \item we define a novel form of \emph{free argumentative exchanges} (FAXs) to characterise explanations amongst agents as illustrated in Figure~\ref{fig:example}; differently from 
    \cite{Rago_23}, these exchanges allow for agents to disagree on what constitutes an attack or support amongst arguments exchanged during debates (and are thus \emph{free});
    \CR{
    this technical novelty empowers the use of \FAXIC s for explanation of image classifiers; }
    \item we instantiate FAXs so that they can serve as the basis for explaining the outputs of image classifiers; 
   \item we provide an implementation of the instantiated FAXs, by adapting the methodology of \cite{kori2022explaining} to allow agents to generate their own arguments;
    \item we evaluate our methodology and implementation quantitatively  and qualitatively, with two types of image classifiers on two datasets;
    for the quantitative evaluation, we use two novel metrics 
    to assess the argumentative quality of the generated debates, and, for comparison with baselines
    , two existing metrics (adapted to our setting) for ascertaining the 
    faithfulness of explanations to the explained classifiers.
\end{itemize}


\section{Related Work}
\label{sec:related_Work}
\begin{figure}
    \centering  \includegraphics[width=1\columnwidth]{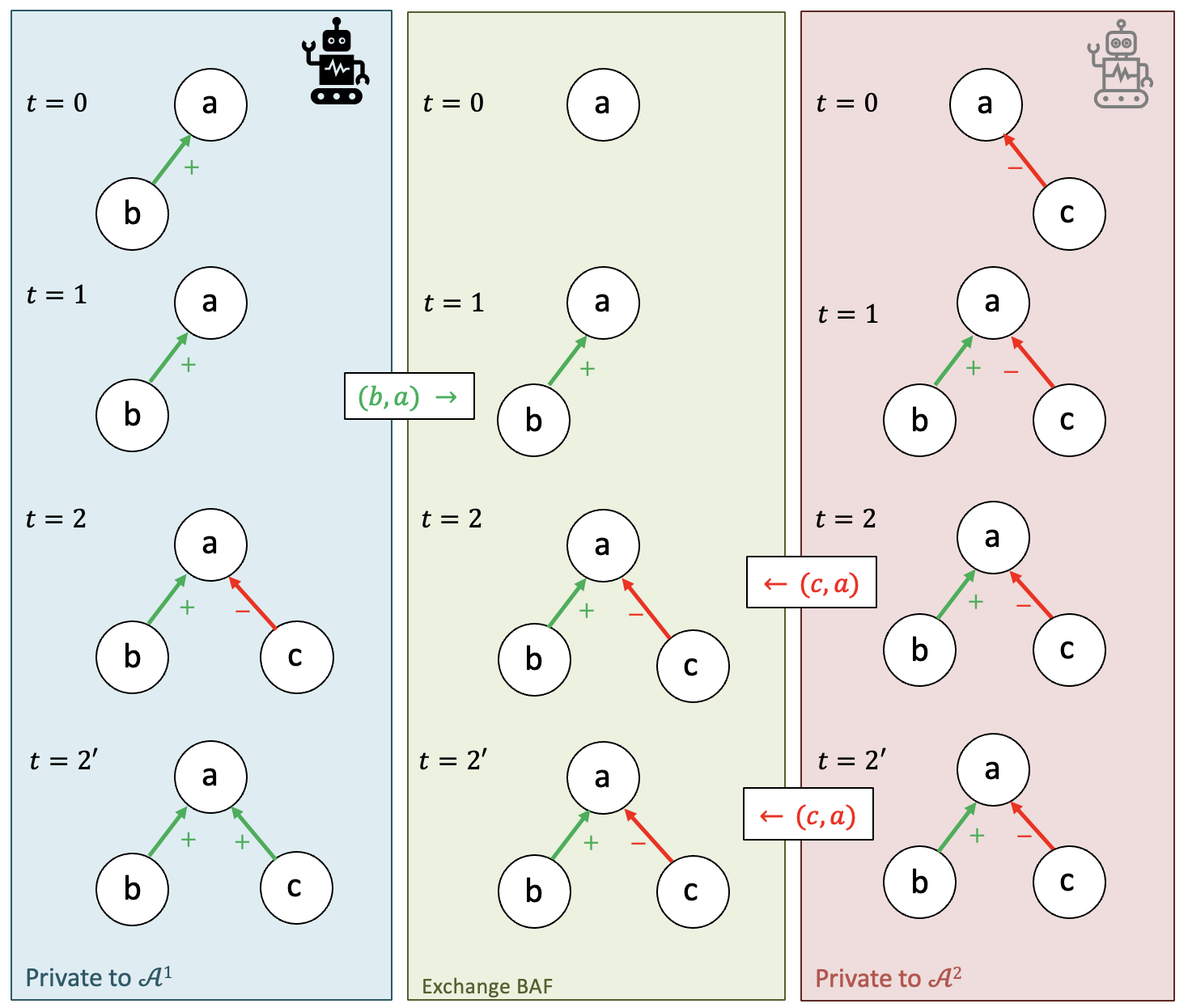}
    \caption{
    Illustration of FAXs (see Example~\ref{ex:illu}). }
    \label{fig:example}
    \vspace{-10pt}
\end{figure}

\paragraph{XAI methods} 
There has been a recent surge in methods advocating for a shift in how we perceive explanations, emphasizing the importance of viewing them as dialogues rather than solely relying on heatmaps or feature attributions, as 
standard in much of the XAI literature~\cite{miller2019explanation,madumal2019grounded,lakkaraju2022rethinking}. 
Within the multi-agent setting,
\cite{Rago_23} proposed an argumentative framework for expressing (interactive) explanations in the form of dialogues. 
Additionally, \cite{debate} demonstrated a debate framework, which was further expanded in \cite{kori2022explaining} to scale, allowing for the extraction of post-hoc explanations in the form of dialogues between two fictional 
agents.
Inspired by \cite{Rago_23}, we define a multi-agent argumentative framework for explanation, 
and we adopt a variant of approach in \cite{kori2022explaining} to implement the framework.
Both \cite{kori2022explaining} and our implementation utilize a surrogate model that faithfully represents the given classifier. 
The use of surrogate models is a standard practice in XAI, as seen e.g. in \cite{yoon2018invase,goyal2019counterfactual,glance}.

The agents in our explanations put forward arguments which can be understood as feature attributions such as 
LIME~\cite{lime}.
However, unlike \cite{lime}, we do not randomly select input regions to mask; instead, 
our agents learn two different strategies to select regions to argue for and against a particular explanandum (input-output).

Our method is also related to
\cite{shitole2021one}, which argues against the rigidity of static and shallow explanations and introduces a new method for compactly visualizing how different combinations of regions in images impact the confidence of a classifier. 
Also, \cite{wang2019deliberative} aims to encourage the capturing of uncertain image regions, 
while \cite{wang2019deliberative} focus on statically capturing ambiguities in an image with respect to the given classifier, we generate both certain and uncertain regions through 
agent interactions in an iterative fashion \cite{Thauvin_24}.

\paragraph{Argumentation methods}
Some research in XAI explores the use of computational argumentation \cite{vcyras2021argumentative}. 
This typically aims to assess specific claims by considering arguments that 
support and/or challenge the claim, as well as their relations within 
argumentative frameworks (AFs). 
These AFs 
may be as in \cite{dung1995acceptability} or 
Bipolar AFs (BAFs).
Broadly, with our XAI 
approach 
we delve into a relatively unexplored area: explaining image classifiers through debates amounting to 
BAFs, which involve interactive gameplay among learning 
agents. 
Other approaches employing AFs for explainable image classification either utilize intrinsically argumentative models, e.g. as in \cite{hamed}, or mirror the mechanics of the model itself, e.g. as seen in  \cite{Purin21}. 
In contrast, our approach focuses on explaining classifiers using latent features through (free) argumentative exchanges.

%


\section{Preliminaries}
\label{sec:prelim}

\subsection{Computational Argumentation Background}

We use (BAFs) \cite{Cayrol_05}, i.e. triples $\langle \Args, \Atts, \Supps \rangle$ such that $\Args$ is a finite set of \emph{arguments}, $\Atts \subseteq \Args \times \Args$ is an \emph{attack} relation and $\Supps \subseteq \Args \times \Args$
is a \emph{support} relation.
For all BAFs in this paper, we assume that $\Atts \!\cap \! \Supps \!= \!\emptyset$.
For 
any $\argument \!\in \! \Args$,  $\Atts(\argument)\!=\!\{ \otherargument \!\in \!\Args | (\otherargument, \argument) \!\in \! \Atts \}$ are the \emph{attackers}  of $\argument$ and  $\Supps(\argument)\!=\!\{ \otherargument \!\in \! \Args | (\otherargument, \argument) \!\in \!\Supps \}$ are the \emph{supporters} of $\argument$.

We use the following notation from \cite{Rago_23}:
given BAFs $\BAF\!=\!\langle \Args, \Atts, \Supps \rangle$, $\BAF'\!=\!\langle \Args', \Atts', \Supps' \rangle$, we say that $\BAF \sqsubseteq \BAF'$ iff $\Args\subseteq \Args'$,
$\Atts\subseteq \Atts'$ and
$\Supps\subseteq \Supps'$; 
also, we use $\BAF' \setminus \BAF$ to denote $\langle \Args' \setminus \Args, \Atts' \setminus \Atts, \Supps' \setminus \Supps \rangle$.
We also say that $\BAF = \BAF'$ iff $\BAF \sqsubseteq \BAF'$ and $\BAF' \sqsubseteq \BAF$, and $\BAF \sqsubset \BAF'$ iff $\BAF \sqsubseteq \BAF'$ but $\BAF \neq \BAF'$. 

As conventional in the literature \cite{Baroni_19}, a BAF $\langle \Args, \Atts, \Supps \rangle$ may be equipped with
\emph{gradual semantics}
$\SF: \Args \rightarrow \Range$
assigning to arguments $\argument\in \Args$ values in 
$\Range$,
which is some set equipped with a pre-order $\leq$ (where, as usual $v < w$ denotes $v \leq w$ and $w \nleq v$).
In line with \cite{Rago_23}, we refer to $\SF$ as \emph{evaluation method} and to $\Range$ as \emph{evaluation range}, to indicate their use by agents
to evaluate arguments internally.
We say 
$\SF$ is \emph{dialectically monotonic} iff, 
as in \cite{Baroni_19}: 
\begin{itemize}
    \item given two BAFs $\BAF = \langle \Args, \Atts, \Supps \rangle$ and $\BAF' = \langle \Args', \Atts', \Supps' \rangle$, where $\Args' = \Args \cup \{ \arga \}$, $\Atts' \cup \Supps' = \Atts \cup \Supps \cup \{ (\arga, \argb) \}$, 
    it is always the case that if $(\arga, \argb) \in \Atts'$, then $\SF(\BAF',\argb) \leq \SF(\BAF,\argb)$, while if $(\arga, \argb) \in \Supps'$, then $\SF(\BAF',\argb) \geq \SF(\BAF,\argb)$; and
    %
    %
    %
    \item given two BAFs $\BAF = \langle \Args, \Atts, \Supps \rangle$ and $\BAF' = \langle \Args', \Atts', \Supps' \rangle$, where for an argument $\arga \in \Args \cap \Args'$, $\Atts'(\arga) = \Atts(\arga)$, $\Supps'(\arga) = \Supps(\arga)$, $\exists \argb \in \Atts'(\arga) \cup \Supps'(\arga)$ such that $\SF(\BAF',\argb) > \SF(\BAF,\argb)$ and $\forall \argc \in \Atts'(\arga) \cup \Supps'(\arga) \setminus \{ \argb \}$, $\SF(\BAF',\argc) = \SF(\BAF,\argc)$
    it is always the case that if $\argb \in \Atts'(\arga)$, then $\SF(\BAF',\arga) \leq \SF(\BAF,\arga)$, while if $\argb \in \Supps'(\arga)$, then $\SF(\BAF',\arga) \geq \SF(\BAF,\arga)$.
\end{itemize}

\CR{The first bullet states that an argument’s strength cannot increase/decrease when a new attacker against/supporter for (respectively) the argument is added, all else being equal; the second bullet states that an argument’s strength cannot increase/decrease when an attacker against/supporter for (respectively) the argument is strengthened, all else being equal.}

Given a BAF $\BAF \!\!=\!\! \langle \Args, \Atts, \Supps \rangle$, for 
$\arga, \argb \!\in \! \Args$,
we let a \emph{path} from $\arga$ to $\argb$ be defined as $(\argc_0,\argc_1), \ldots, (\argc_{n-1}, \argc_{n})$ for some $n\!>\!0$ (
the \emph{length} of the path) where $\argc_0 = \arga$, $\argc_n = \argb$ and, for any $1 \leq i \leq n$, $(\argc_{i-1}, \argc_{i}) \in \Atts \cup \Supps$. We also use $\argpaths(\arga,\argb)$ to denote the set of all paths from $\arga$ to $\argb$ (leaving the BAF implicit), and use $|p|$ for the length of path $p$. Also, we may see paths as sets of pairs. 
Then, 
as in \cite{Tarle_22,Rago_23},
we say that $\BAF$ is a \emph{BAF for explanandum $\arge\in \Args$} iff i.) $\nexists (\arge,\arga) \in \Atts \cup \Supps$; ii.) $\forall \arga \in \Args \setminus \{\arge\}$, there is a path from $\arga$ to $\arge$; and 
iii.) $\nexists \arga \!\in\! \Args$ with a path from $\arga$ to $\arga$.

\subsection{Image Classification Set-up}

Consider a dataset $\mathcal{D} \subseteq \mathcal{X} \times \mathcal{Y}$, where $\mathcal{X} \in \mathbb{R}^{h \times w \times c}$ represents a vector space with dimensions $h \times w$ corresponding to an image and $c$ channels $(c \geq 1)$. 
The label space $\mathcal{Y} = {1, \dots, N}$ consists of $N \geq 2$ classes.
Following conventions in image classification \cite{huang2017densely,he2016deep}, we consider an image classifier trained on $\mathcal{D}$, comprising a feature extractor $f: \mathcal{X} \rightarrow \mathcal{Z}$ and a feature classifier $g:\mathcal{Z} \rightarrow \mathcal{Y}$. 
The feature extractor $f$ maps the observational space $\mathcal{X}$ to a continuous latent space $\mathcal{Z} \subseteq \mathbb{R}^{m\times d}$, where each element of $\mathcal{Z}$ represents a set of latent features with $m$ elements, each of length $d$.
Given an input $x \in \mathcal{X}$, the image classifier orders the classes in $\mathcal{Y}$ based on $g(f(x))$, predicting the top-class $y$ as the output (with an abuse of notation we often write $g(f(x))=y$).
Consistent with prior works \cite{van2017neural,kori2022explaining}, we assume the existence of a set $\mathcal{Z}^i\subset \mathcal{Z}$ of features associated with each class $i \in \mathcal{Y}$. 
These are referred to as the \emph{class-$i$-specific features}. Up to Section~\ref{sec:set-up}, we defer discussion on how these class-specific features are obtained.

 %
%



\section{Agents and 
Explanations}

Agents are represented as \emph{private triples for explananda}, a notion adapted from \cite{Rago_23}, as follows.

\begin{definition}
\label{def:agent}
An agent $\AgentI$ is a \emph{private triple for an explanandum $\arge$}, amounting to
 $(\RangeI,\BAFi,\SFi)$ where: 
    \begin{itemize}
        \item  $\RangeI = \RPosI \cup \RNegI$ is an evaluation range, referred to as $\AgentI$'s \emph{private 
        evaluation range},
         where $\RPosI$ (referred to as \emph{positive evaluations})
         and $\RNegI$ (referred to as \emph{negative evaluations}) are disjoint and for any $v_+ \in \RPosI$ and $v_- \in \RNegI$, $v_+ > v_-$;
        \item  $\BAFi \!=\! \langle \ArgsI, \AttsI, \SuppsI \rangle$ is a BAF for $\arge$, referred to as $\AgentI$'s \emph{private BAF};
        \item  $\SFi$  is an evaluation method, referred to as $\AgentI$'s \emph{private 
        evaluation method}, such that, for any BAF $\BAF = \langle \Args, \Atts, \Supps \rangle$ and, for any $\arga \in \Args$,
         $\SFi(\BAF,\arga) \in \RangeI$.
    \end{itemize}
\end{definition}

Here, we use BAFs instead of quantitative BAFs~\cite{Baroni_18} as in \cite{Rago_23}, and rely upon evaluation ranges split into two, rather than three, partitions as in \cite{Rago_23} (so we disregard the neutral partition in \cite{Rago_23}). 
The threshold between the two partitions is seen as the point where an agent ``changes their mind'' on the explanandum, and may correspond to a classifier's decision boundary, as we will see later.
We assume, for the remainder of this section, that any evaluation method is dialectically monotonic (as defined in Section \ref{sec:prelim}).
\CR{We exemplify our notion of an agent in a simple, generic setting below (see Section~\ref{sec:EX-IC} for instantiations for image classification).}

\begin{example}
\label{ex:agents}
    An agent $\Agents^1$ is a private triple for an explanandum $a$, amounting to $(\Range^1,\BAF^1,\SF^1)$ where: $\Range^1 = [0,1]$ with $\RNeg^1 = [0,0.6[$ and $\RPos^1 = [0.6,1]$; $\BAF^1 = \langle \Args^1, \Atts^1, \Supps^1 \rangle$ such that $\Args^1 = \{ a, b \}$ with arguments
    $a$: \emph{we should eat at this pizzeria} and
    $b$: \emph{it is highly recommmended}, 
    $\Supps^1(a) = \{ (b, a) \}$ and 
    $\Atts^1(a) = \emptyset$; and $\SF^1$ is some dialectically monotonic semantics, which in this case could give, for example $\SF^1(\BAF^1,a)= 0.75$ and $\SF^1(\BAF^1,b)=0.5$ (given the asymmetric set-up of $a$'s attackers and supporters). 
    Here we can see that $\Agents^1$'s positive reasoning for the explanandum overcomes the (absent) negative reasoning against it, resulting in its strength being above the threshold of $0.6$ and thus gives a positive evaluation.
\end{example}

We define a novel form of \emph{argumentation exchanges} amongst agents
, which will serve, later, as explanations:\footnote{ 
Throughout, we denote with $[k]$ the set $\{0,1,\ldots,k\}$, with $]k]$ 
$\{1,\ldots,k\}$, 
etc.}

\begin{definition}
\label{def:FAXs}
    A \emph{free argumentative exchange (FAX) for an explanandum $\arge$ amongst agents} $\Agents$, where $|\Agents|=m \geq 2$ and each agent in $\Agents$ is a private triple for $\arge$, is a tuple:
    $$\langle \BAFx_0, \ldots, \BAFx_{n}, \Agents_0, \ldots,\Agents_n, \SpeakerM 
    \rangle \quad (n \geq 0)$$ 
    where, for all 
    timesteps $t \in [n]$, $\BAFx_t$ (referred to as the \emph{exchange BAF at step} $t$) is a BAF for $\arge$ and $\Agents_t$ is a set of $k$ agents, all private triples for  $\arge$, such that:
    \begin{itemize}
        \item  $\BAFx_0 = \langle \ArgsX_0, \AttsX_0, \SuppsX_0 \rangle$ is a BAF such that: 
        \begin{itemize}
            \item $\ArgsX_0 = \{ \arge \}$;
            %
            $\AttsX_0 = \emptyset$;
            %
            $\SuppsX_0 = \emptyset$;
           
        \end{itemize}
         \item $\Agents_0=\Agents$;
    \end{itemize} 
    and, at timestep 
    $t \in ]n]$, letting
        $\BAFx_* = \langle \ArgsX_*, \AttsX_*, \SuppsX_* \rangle = \BAFx_t \setminus \BAFx_{t-1}$:
    \begin{itemize}

        \item $\BAFx_{t} \sqsupseteq \BAFx_{t-1}$, where $\forall (a,b) \in \AttsX_*$, $\exists j \in ]m]$ such that $(a,b) \in \Atts^j_{t-1}$ and $\forall (c,d) \in \SuppsX_*$, $\exists k \in ]m]$ such that $(c,d) \in \Supps^k_{t-1}$;
        
        \item  $\Agents_t$ is a set of private triples $(\Range^i,\BAF^i_t,\SF^i)$ for $\arge$, one for each  
        $i\in
        ]m]$, where  $\BAFi_{t} \sqsupseteq \BAFi_{t-1}$ and $\ArgsX_* \subseteq \ArgsI_t$ and $\AttsX_* \cup \SuppsX_* \subseteq \AttsI_t \cup \SuppsI_t$; 
    \end{itemize}
    $\SpeakerM$
    \CR{is the \emph{contributor mapping},} such that, for  $\BAFx_n = \langle \ArgsX_n, \AttsX_n, \SuppsX_n \rangle$, for every $(\arga,\argb) \in \AttsX_{n} \cup \SuppsX_{n}$: $\SpeakerM((\arga,\argb))=(
    i,t)$ 
    with 
    $i \in ]m]$ and  
     $t \in ]n]$.
\end{definition}
FAXs thus allow agents to add attacks or supports from their private to the exchange BAF, which are then incorporated to all other agents' frameworks as some form of relation. Note that it is only the BAFs which change with the timestep, not, for example, the evaluation ranges or methods.
Intuitively, the contributor mapping returns, for each attack/support pair in the final exchange BAF,  the agent who contributes the pair as well as the timestep at which the pair was contributed. Note that,  by defining $\SpeakerM$ as a mapping, we impose that there is a single contributor and timestep for each attack/support pair in the final exchange BAF, and therefore pairs cannot be introduced multiple times in FAXs
.
Note that FAXs are 
variants of argumentative exchanges (AXs) in \cite{Rago_23}:
whereas in AXs agents are equipped with quantitative BAFs, in FAXs they are equipped with BAFs;  
and, 
whereas in AXs agents are assumed to share a \emph{lingua franca} of arguments' attackers and supporters  (so that if an argument attacks or supports another for an agent it does so for all others), in FAXs attackers/supporters for an agent may be supporters/attackers for another
, witnessing the `free'' nature of FAXs.
Next, we illustrate \CR{(in the earlier simple, generic setting)} why this may be useful.

\begin{example}
\label{ex:FAXs}
    Continuing from Example \ref{ex:agents}, a second agent 
    $\Agents^2$ is a private triple for $a$, amounting to $(\Range^2,\BAF^2,\SF^2)$ where: $\Range^2 = \Range^1$; $\BAF^2 = \langle \Args^2, \Atts^2, \Supps^2 \rangle$ such that $\Args^2 = \{ a, c \}$ with argument $c$: \emph{there is pineapple on the pizza},  $\Atts^2(a) = \{ (c, a) \}$ and
    $\Supps^2(a) = \emptyset$; 
    and $\SF^2$ is some dialectically monotonic semantics, which in this case could give, 
    e.g., $\SF^2(\BAF^2,a)\!=\! 0.25$ and $\SF^2(\BAF^2,c)\!=\!0.5$ (again given the asymmetric set-up of $a$'s attackers and supporters).
    Here 
    $\Agents^2$'s negative reasoning for the explanandum overcomes the (absent) positive reasoning against it, resulting in its strength being below the threshold of $0.6$ and thus gives a negative evaluation.
    A FAX for $a$ amongst agents $\Agents = \{ \Agents^1, \Agents^2 \}$ is then a tuple
    $\langle \BAFx_0, \BAFx_1, \BAFx_2, \Agents_0, \Agents_1, \Agents_2, \SpeakerM 
    \rangle$ such that:
    \begin{itemize}
        \item $\BAFx_{1} \setminus \BAFx_{0} = \langle \{ b \}, \emptyset, \{ (b, a) \} \rangle$, where $\SpeakerM((b,a)) = (1, 1)$, and thus $\Agents_1 \neq \Agents_0$, where $\Agents_1^1 = \Agents_0^1$ and $\Agents_1^2 \neq \Agents_0^2$ with $\BAF^2_{1} \setminus \BAF^2_{0} = \langle \{ b \}, \emptyset, \{ (b, a) \} \rangle $;
        \item $\BAFx_{2} \setminus \BAFx_{1} = \langle \{ c \}, \{ (c, a) \}, \emptyset \rangle$, where $\SpeakerM((c,a)) = (2, 2)$, and thus $\Agents_2 \neq \Agents_1$, where $\Agents_2^2 = \Agents_1^2$ and $\Agents_2^1 \neq \Agents_1^1$ with $\BAF^1_{2} \setminus \BAF^1_{1} = \langle \{ c \}, \emptyset, \{ (c, a) \} \rangle $.
    \end{itemize}
    Here, support $(b,a)$ provided at timestep $2$ by $\Agents^1$ is learnt by $\Agents^2$ (as is conventional in AXs \cite{Rago_23}) as a support, indicating that the agents agree on the relation.
    Then, at timestep $2$, $\Agents^2$ provides the attack $(c,a)$, e.g. because pineapple on a pizza is anathema in Italy.
    However, $\Agents^1$ learns the relation $(c, a)$ as a support, indicating that they considered $c$ to be providing reasoning for $a$, e.g. because $\Agents^1$ likes pineapple on 
    pizza.
    This 
    shows how FAXs, differently to AXs, allow for differences in the way 
    agents interpret relations.
    Given that $\SF^1$ and $\SF^2$ are dialectically monotonic
    , we know that $\SF^1(\BAF^1_0, \arga) = \SF^1(\BAF^1_1, \arga) \leq \SF^1(\BAF^1_2, \arga)$ and $\SF^2(\BAF^2_0, \arga) \leq \SF^2(\BAF^2_1, \arga) = \SF^2(\BAF^2_2, \arga)$, respectively.
\end{example}

We can define a notion to restrict FAXs so that agents therein can be deemed to share a lingua franca, 
as follows.
\todoE{test how often this restriction is satisfied in practice?}

\begin{definition}
\label{def:linguafranca}
    Given a FAX $F$ for $e$ amongsts $\Agents$, with 
    $F=\langle \BAFx_0, \ldots,$ $ \BAFx_{n}, \Agents_0, \ldots,\Agents_n, \SpeakerM 
    \rangle$, $\Agents_n^i \in \Agents_n$ such that $\Agents_n^i=(\Range^i,\BAF^i_n,\SF^i)$ and  $\BAFi_n = \langle \ArgsI_n, \AttsI_n, \SuppsI_n \rangle$,
    we say that $F$ has an \emph{effective lingua franca} iff    
    
    $(\bigcup_{\Agents^i_n \in \Agents_n} \AttsI_n ) \cap (\bigcup_{\Agents^i_n \in \Agents_n} \SuppsI_n ) \!=\! \emptyset$.
\end{definition}

It can be seen that the FAX in Example \ref{ex:FAXs} does not have a lingua franca since the agents disagree on the relation $(c,a)$.
%
Typically, FAXs begin because there is a \emph{conflict} between agents, amounting to a different \emph{stance} on the explanandum: 

\begin{definition}
\label{def:stance} \label{def:conflict}
    Given an agent  $\AgentI$, i.e. a private triple $(\RangeI,\BAFi,\SFi)$    (for some explanandum $\arge$)  with $\BAFi = \langle \ArgsI, \AttsI, \SuppsI \rangle$, for any $\arga \in \ArgsI$, let $\AgentI$'s \emph{stance on $\arga$} be defined as $\StanceI(\BAFi,\arga) = +$ (\emph{positive stance}) iff $\SFi(\BAFi, \arga) \in \RPosI$, and $\StanceI(\BAFi,\arga) = -$ (\emph{negative stance}) otherwise.
Then, a set $\Agents$ of agents/ private triples for explanandum 
$\arge$ is \emph{in conflict wrt $\arge$}  iff there are two or more agents in $\Agents$  with different stances on $\arge$.
\end{definition}

Note that this
 notion of stance is adapted from \cite{Rago_23}, ignoring the neutral stance therein.

We see FAXs as means to lead to resolution of initial conflicts, by allowing agents to argue while identifying and filling any gaps in their beliefs (represented by their private BAFs). We adapt the following from \cite{Rago_23}, to characterise FAXs successfully leading to resolution (or not), from an initial conflict.

\begin{definition}\label{def:resolution}
 Let  $\Agents$ be a set of agents/private triples for explanandum $\arge$.
 Let $\Agents$ be in conflict wrt 
 $\arge$. 
Let $F\!=\!\langle \BAFx_0, \ldots, \BAFx_{n}, \Agents_0, \ldots,\Agents_n, \SpeakerM 
    \rangle$ 
    be a FAX for $\arge$ amongst 
    $\Agents$. Then: 
    
    \begin{itemize}
        \item $F$ is 
        \emph{unresolved at timestep $t$}, for 
        $t \in ]n]$, iff $\Agents_t$ is in conflict wrt $\arge$, and  \emph{resolved at $t$} otherwise;
         \item $F$ is \emph{unresolved} iff it is unresolved at every timestep 
        $t \in ]n]$
        ;
        \item $F$ is \emph{resolved} iff it is resolved at timestep $n$
        . 
       \end{itemize}
    \end{definition}

\begin{example}
\label{ex:(non)resolution}
  Continuing from Example \ref{ex:FAXs}, we know that $\SF^1(\BAF^1_0, \arga) \in \RPos^1$ and $\SF^2(\BAF^2_0, \arga) \in \RNeg^2$, and so $\Stance^1(\BAF^1_0, \arga) = +$ $\Stance^2(\BAF^2_0, \arga) = -$, meaning $\Agents$ is in conflict wrt $\arga$. 
  At timestep 1, let us assume that, although $\SF^2(\BAF^2_0, \arga) \leq \SF^2(\BAF^2_1, \arga)$, it remains the case that $\SF^2(\BAF^2_1, \arga) \in \RNeg^2$. Since $\SF^1(\BAF^1_0, \arga) = \SF^1(\BAF^1_1, \arga) \in \RPos^1$, we can see that the FAX is unresolved at timestep 1. Then, since $\SF^1(\BAF^1_1, \arga) \leq \SF^1(\BAF^1_2, \arga)$, thus $\SF^1(\BAF^1_2, \arga) \in \RPos^1$ and $\SF^2(\BAF^2_1, \arga) = \SF^2(\BAF^2_2, \arga)  \in \RNeg^1$, we know that $\Stance^1(\BAF^1_2, \arga) = +$, $\Stance^2(\BAF^2_2, \arga) = -$ and thus the FAX is unresolved.
  Meanwhile, if $\Agents^1$ had interpreted $(\argc,\arga)$ as an attack, let us say at an alternate $t=2'$, since $\SF^1$ is dialectically monotonic we know that it would have been the case that $\SF^1(\BAF^1_1, \arga) \geq \SF^1(\BAF^1_2, \arga)$, and the FAX may have been resolved if $\SF^1(\BAF^1_{2'}, \arga) \in \RNeg^1$, thus giving $\Stance^1(\BAF^1_{2'}, \arga) = \Stance^2(\BAF^2_{2'}, \arga) = -$.
\end{example}

When FAXs are used for conflict resolution,
exchange BAFs therein can be seen as \emph{explanations}, in that they unearth the reasoning behind the resolution or otherwise of the conflict amongst the agents, with evidence that the explanandum is ``correct'' or not (when the FAX is resolved, depending on the final stance of all agents) or why it cannot be deemed ``correct'' or otherwise (when the FAX is unresolved).
For illustration,
in the first FAX ($t=2$) in Example~\ref{ex:(non)resolution}, the agents do not share the same stance on the explanandum due to their differing interpretations of $(\argc,\arga)$,
whereas in the second FAX ($t=2'$), they agree on both this relation being an attack and on their final stances on the explanandum.

When using FAX for explaining image classification (from Section~\ref{sec:EX-IC}), we will restrict attention to special forms of FAXs, as follows.

\begin{definition}
A \emph{strictly interleaved FAX} for  $\arge$ amongst 
    $\Agents$ is a FAX
$\langle \BAFx_0, \ldots, \BAFx_{n}, \Agents_0, \ldots,\Agents_n, \SpeakerM 
    \rangle$ 
    for  $\arge$ amongst 
    $\Agents$ 
such that

    \begin{itemize}
        \item $\Agents=\{\Agents^1,\Agents^2\}$ (i.e. with $|\Agents|=2$);
    \item for each timestep 
    $t \in ]n]$ there exists exactly one $(\arga,\argb)$ \todoE{is this really happening in the implementation? to be discussed after deadline} such that $\SpeakerM((\arga,\argb))=(
    k,t)$ and 
        if $t$ is odd then 
        $k=
        1$;
        \CR{else $k=2$.}
        %
    \end{itemize} 
    If $n$ is even,
    then the 
    FAX is \emph{equal opportunity}.
    \end{definition}

The FAX in Example~\ref{ex:FAXs} is equal opportunity strictly interleaved.
Intuitively,  in a strictly interleaved FAX agents take turns, contributing one attack/support at a time 
and making the same number of contributions, 
to have the same chances at persuading the other
.

\section{
Explanations for Image Classification}
\label{sec:EX-IC}

We see explanations for image classification as (exchange BAFs of) equal opportunity strictly interleaved FAXs amongst:
\begin{itemize}
    \item 
$\Agents^1$ (the \emph{proponent}, arguing for the class in $\mathcal{Y}$ predicted by the underlying image classifier for input image $x\in \mathcal{X}$ of interest), and

\item $\Agents^2$ (the \emph{opponent}, arguing against the predicted class). 
\end{itemize}
For simplicity, we restrict attention to the top two classes in the ordering determined by the image classifier ($g$, see Section~\ref{sec:prelim}) for the input image
, so that 
$\Agents^2$
argues against the predicted (top) class by arguing for the 
second best in the ordering.
In the remainder of the paper we assume, without loss of generality, that class $1$ is the predicted class for $x$ and class $2$ is the second best. In line with the machine learning literature, we also  refer to class 1 as $y$.

In our approach to explaining image classification using FAXs, each agent argues for a particular class, and thus class-specific features (see Section~\ref{sec:prelim}) play the role of arguments exchanged between agents.
The arguments put forward by the proponent and opponent can be seen as playing the role of positive and negative, respectively, feature attributions as in some XAI literature \cite{lime,SHAP,deeplift,selvaraju2017grad}. However, intuitively, the exchange BAF in the FAX can also convey the reasoning of the classifier.
We will provide an evaluation of FAXs as explanations of image classifiers in Section~\ref{sec:results}. Here, we focus on instantiating the generic FAXs for our purposes.

Concretely, we assume that the two agents explaining the output of an underlying image classifier by virtue of a 
FAX 
can leverage on 
\emph{class-specific classifiers} $q^1:2^{\mathcal{Z}^1\cup\mathcal{Z}^2}\rightarrow [0
,1]$ and $q^2:2^{\mathcal{Z}^1\cup\mathcal{Z}^2}\rightarrow [0
,1]$, 
allowing the agents to evaluate sets of their own and other agents' arguments (in other words, amounting to their private evaluation methods). 
Until 
Section~\ref{sec:set-up}, we ignore how
these class-specific classifiers can be learnt, alongside the class-specific features, so that they are faithful to the original image classifier being explained (see Section~\ref{sec:eval}), 
which is crucial to guarantee that FAXs as explanations are also faithful. 
We also refer to  the class-specific classifiers as \emph{private classifiers} (with $\mathcal{Z}^{1}$  and 
$\mathcal{Z}^{2}$
also referred to  as \emph{private features})
.
Finally, given that we aim at explaining the output of the image classifier, we use 
$(x,y)$ as the \emph{explanandum} \arge.

We now instantiate the general Definition~\ref{def:agent}, describing agents, to capture proponent and opponent for image classification. 

\begin{definition}
\label{def:imagetriple}
    Let $\indexAg\in \{1,2\}$.  Then, the \emph{
    initial image classification agent} 
    is a private triple $(\RangeW,\BAFW,\SFW)$ for $(x,y)$ such that: 
    \begin{itemize}
        \item $\RangeW \!=\! [0,1]$ is the  \emph{agent's evaluation range}, with $\RNeg^\indexAg\! =\! [0,\tau)$ and $\RPos^\indexAg \!=\! [\tau,1]$, for some \emph{threshold} $\tau \in [0,1]$; 
        \item $\BAFW= \langle \ArgsW, \AttsW, \SuppsW \rangle$ is a BAF where if $i=1$ then
        \begin{itemize}
            \item $\ArgsW \subseteq \{(x,y)\} \cup \mathcal{Z}^1$ such that $(x,y) \in \ArgsW$,
            \item $\AttsW = \emptyset$,
            \item $\SuppsW \subseteq \mathcal{Z}^1 \times \{(x,y)\}$; 
        \end{itemize}
        and if $i=2$ then
        \begin{itemize}
            \item $\ArgsW \subseteq \{(x,y)\} \cup \mathcal{Z}^2$ such that $(x,y) \in \ArgsW$,
            \item $\AttsW \subseteq \mathcal{Z}^2 \times \{(x,y)\}$,
            \item $\SuppsW = \emptyset$; 
        \end{itemize}
        \item $\SFW: \ArgsW \rightarrow \RangeW$ is such that:   
        \begin{itemize}
            \item $\SFW(\BAFW,(x,y))= \qAgentW(\AttsW((x,y))\cup \SuppsW((x,y)))$;
            \item for $z \in \ArgsW \setminus \{(x,y)\}$, 
            $\SFW(\BAFW,z)=q^{\indexAg}(\{ z \})$.   
        \end{itemize}
    \end{itemize}
\end{definition}

\begin{corollary}
\label{lem:tau}
    \CR{If $\SF^1(\BAF^1,\!(x,y)) \!\!\neq\!\! \SF^2(\BAF^2,\!(x,y))$, i.e. 
    $\qAgentI( \Supps^1((x,y))) \!$ $\neq\! \qAgentJ(\Atts^2((x,y)))$, then $\exists \tau \in [0,1]$ such that $\Agents$ is in conflict.}
\end{corollary}

Here, we use specialised notions compared to those in Definition~\ref{def:agent}. 
Specifically, the private BAFs are ``shallow'' and acyclic, with all attacks and supports from private features (of either agent) to the explanandum.
The evaluation ranges are divided into positive and negative evaluations by a threshold $\tau$
, which, as Lemma \ref{lem:tau} demonstrates, can be guaranteed to provide a dividing line between the two classes if the agents' class-specific features 
have an effect on the private classifiers' outputs.
The evaluation method is defined in terms of the class-specific classifier of the agent: the evaluation of the explanandum is given by the classifier applied to its attackers and supporters (which are class-specific
features), while the evaluation of a class-specific feature is given by the private classifier applied to this feature only. 
To obtain explanations for image classifiers, we will use FAXs where agents 
update their private BAFs guided by their private classifiers, using the following ``learning strategy''. 


\begin{definition}
\label{def:imagemonotonic}
    Let  $\langle \BAFx_0, \ldots, \BAFx_{n}, \Agents_0, \ldots,\Agents_n, \SpeakerM 
    \rangle$ be an equal opportunity strictly interleaved FAX for explanandum $(x,y)$ amongst initial image classification agents $\Agents = \{ \Agents^1, \Agents^2 \}$.  Then, for $i, j \in \{ 1, 2 \}$ with $i \neq j$,
     $\Agents^i$ adopts a \emph{dialectically monotonic learning strategy} iff
     at timestep 
        $t \in ]n]$, if $\exists (z,(x,y))$ such that $\SpeakerM((z,(x,y)))=(
        i,t)$, then $\BAF^i_t = \BAF^i_{t-1}$; otherwise $\SpeakerM((z,(x,y)))=(
        j,t)$ and $\BAF^i_t = \langle \Args^i_t, \Atts^i_t, \Supps^i_t \rangle$ is such that: 
        \begin{itemize}
            
            \item 
            $(z,(x,y)) \!\in \! \Atts^i_t \!\setminus \!\Atts^i_{t-1}$ iff $q^{i}(\Atts^i_{t-1}((x,y)) \cup \Supps^i_{t-1}((x,y)) \cup \{ z \}) - q^{i}(\Atts^i_{t-1}((x,y)) \cup \Supps^i_{t-1}((x,y))) < 0$;
            \item 
            $(z,(x,y)) \!\in \! \Supps^i_t \setminus \Supps^i_{t-1}$ iff $q^{i}(\Atts^i_{t-1}((x,y)) \cup \Supps^i_{t-1}((x,y)) \cup \{ z \}) - q^{i}(\Atts^i_{t-1}((x,y)) \cup \Supps^i_{t-1}((x,y))) \geq 0$;
        \end{itemize}
        and $\SF^i$ is such that:
        \begin{itemize}
            \item $\SF^i(\BAF^i_t,(x,y))= \qAgentW(\Atts^i_{t-1}((x,y)) \cup \Supps^i_{t-1}((x,y)) \cup \{ z \})$;
            \item $\forall z' \in \Args^i_t \cap \Args^i_{t-1}$, $\SF^i(\BAF^i_{t},z') = \SF^i(\BAF^i_{t-1},z')$; and 

            \item $\SF^i(\BAF^i_{t},z) = |q^{i}(\Atts^i_{t-1}((x,y)) \cup \Supps^i_{t-1}((x,y)) \cup \{ z \}) -$ 
            
            $ q^{i}(\Atts^i_{t-1}((x,y)) \cup \Supps^i_{t-1}((x,y)))|$.
        \end{itemize}
\end{definition}

In the remainder, we refer to equal opportunity strictly interleaved FAXs where both agents adopt a dialectically monotonic learning strategy simply as \FAXIC s.
%
Note that, in \FAXIC s, the initial private BAF of each agent 
contains all 
arguments representing the agent's private features. 
Then, any arguments learned from the other agent's contributions are such that their addition to the agent's private BAF results in dialectically monotonic behaviour (due to the agent's characterisation of the relation as an attacker or supporter), in the spirit of \cite{Rago_22}.
This naturally leads to
:

\begin{proposition}
    In any 
    \FAXIC\ for an explanandum $(x,y)$ amongst agents $\Agents = \{ \Agents^1, \Agents^2 \}$, 
     for any $\Agents^i \in \Agents$, 
     $\SFi$ is dialectically monotonic. 
\end{proposition}

\begin{proof}
    From Definition \ref{def:imagetriple}, 
    for any $\Agents^i \in \Agents$ and any timestep $t$, $\BAF^i_t$ is flat, i.e. it consists of the explanandum $(x,y)$ and its attackers or supporters.
    It can then be seen from Definition \ref{def:imagemonotonic} that the strength of any argument $a \in \Args^i \setminus \{ (x,y) \}$ does not change with the timestep (its attackers and supporters remain empty throughout), i.e. for any timesteps $t>0$, it is necessarily the case that $\SF^i(\BAF^i_t,a) = \SF^i(\BAF^i_{t-1},a)$. Thus, the last bullet of the definition of dialectical monotonicity (see Section \ref{sec:prelim}) is satisfied.
    Meanwhile, the first bullet can only apply to arguments added which attack or support $(x,y)$, since this is the only argument which is attacked or supported, by Definition \ref{def:imagetriple}.
    We can see from Definition \ref{def:imagemonotonic} that for any $t$ such that there exists some $(z,(x,y)) \in \Atts^i_t \setminus \Atts^i_{t-1}$, it is necessarily the case that $q^{i}(\Atts^i_{t}((x,y)) \cup \Supps^i_{t}((x,y))) < q^{i}(\Atts^i_{t-1}((x,y)) \cup \Supps^i_{t-1}((x,y)))$, and thus, by Definition \ref{def:imagetriple}, $\SF^i(\BAF^i_t,(x,y)) < \SF^i(\BAF^i_{t-1},(x,y))$.
    Again from Definition \ref{def:imagemonotonic}, we can see that for any $t$ such that there exists some $(z,(x,y)) \in \Supps^i_t \setminus \Supps^i_{t-1}$, it is necessarily the case that $q^{i}(\Atts^i_{t}((x,y)) \cup \Supps^i_{t}((x,y))) \geq q^{i}(\Atts^i_{t-1}((x,y)) \cup \Supps^i_{t-1}((x,y)))$, and therefore, by Definition \ref{def:imagetriple}, $\SF^i(\BAF^i_t,(x,y)) \geq \SF^i(\BAF^i_{t-1},(x,y))$. 
    Thus, the property is satisfied in both cases for the first bullet, and dialectical monotonicity is satisfied overall.
\end{proof}

\CR{This result sanctions that our choices in instantiating the semantics for \FAXIC s leads to explanations that are dialectically monotonic, which has been identified as an important property by argumentation practitioners~\cite{Amgoud_18,Baroni_19,Potyka_21} and found to be intuitive by humans (as shown for a form of dialectical monotonicity in \cite{dax}).}

\begin{figure*}
    \centering
    \includegraphics[width=0.95\textwidth]{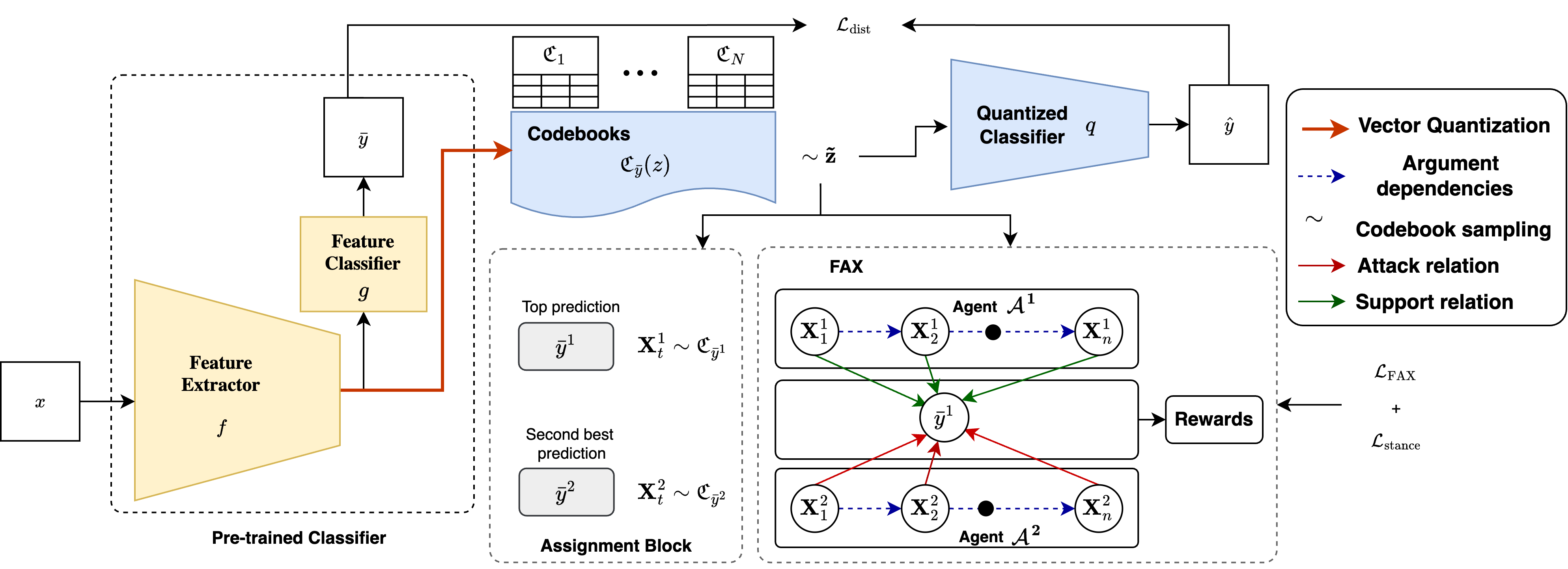}
    \caption{Overview of our implementation (argument dependencies are temporal).}
    \label{fig:overview}
\end{figure*}

\section{
Implementation}
\label{sec:set-up}

In this section, we detail our methodology (overviewed in Figure~\ref{fig:overview}) for obtaining and  evaluating empirically \FAXIC s 
with data for image classification.
Specifically, we detail 
how class-specific (private) features  can be obtained (Section~\ref{sec:feat}); and how class-specific agents (including their private classifiers and their policies for contributing to \FAXIC s) are learnt and deployed (Section~\ref{sec:agent}).
Note that, while \FAXIC s are amongst two agents only (for the top- and second-best predicted classes by the classifier), given that different inputs will result in different predictions we need to develop, at training time, all agents, with their private features and classifiers. 

\subsection{
Class-specific discrete features}
\label{sec:feat}

We obtain 
these by 
simultaneously training, similarly 
to \cite{kori2022explaining}: 
\begin{itemize}
    \item 
$N$ \emph{codebooks} 
$\mathfrak{C}_1, \ldots, \mathfrak{C}_N$ (one per class in $\mathcal{Y}$); for each $i \in \{1, \ldots, N\}$,
$\mathfrak{C}_i \in \mathbb{R}^{\tilde{n} \times d}$   corresponds to
$\mathcal{Z}^i$ (see Section~\ref{sec:prelim});\footnote{We assume  for simplicity that $M^i=\tilde{n}$  for all $i\in \{1, \ldots, N\}$, namely all agents have the same number of class-specific features.} for $z=f(x)$, if $\bar{y}$ is the  top-class predicted
by $g(z)$, we use $\mathfrak{C}_{\bar{y}}(z)$ to stand for the specific discrete features in 
$\mathfrak{C}_{\bar{y}}$ corresponding to $z$  (we use 
$\tilde{z}$ to stand for $\mathfrak{C}_{\bar{y}}(z)$ if clear
);
\item 
a \emph{quantized classifier} $q:\mathfrak{C} \rightarrow \mathcal{Y}$,
for $\mathfrak{C}=\mathfrak{C}_1\cup \ldots \cup \mathfrak{C}_N$, distilling the knowledge of the feature classifier $g$ so that
$q(\tilde{z}
)$ approximates $g(z)$.
\end{itemize}
 %
Intuitively, 
a codebook is a collection of averaged patterns or ``concept representations'' that summarize the key features 
of data \cite{van2017neural}. 

\paragraph{Training.} To 
get the codebooks, 
we draw inspiration from \cite{van2017neural}. Intuitively, for each $(x,y) \!\in\! \mathcal{D}$, with $z\!=\!f(x)$, for $\bar{y} \!=\! \texttt{argmax}(g(z))$ (i.e. $\bar{y}$ is the top-class predicted by the feature classifier $g$ for $z$), we aim at
deterministically mapping the elements of $z$ to the nearest elements of $\mathfrak{C}_{\bar{y}}$ using some convex distance function $\delta$, as follows:


\vspace*{-0.2cm}
\begin{equation}
    \tilde{z} = \{\argmin_{\tilde{k} \in \mathfrak{C}_{\bar{y}}} \delta (z_k, \tilde{k})| z_k \in z \} 
    \label{eqn:codebooks}
\end{equation}

To learn  
this in an end-to-end fashion, we use the Gumble sampling procedure from \cite{jang2016categorical},  resulting in $\tilde{z}$ as a projection of the continuous features in $z$ onto every element of codebook 
$\mathfrak{C}_{\bar{y}}$, corresponding to pairwise similarity scores between $z$ and all $\tilde{n}$ codebook features $\mathfrak{C}_{\bar{y}}$.
The resulting quantization objective for training is described as follows (see \cite{jang2016categorical} for details):
\begin{equation}
    \mathcal{L}_{\textrm{quant}} = \sum \texttt{softmax}(\tilde{z}) \left(\texttt{logsoftmax}(\tilde{z})\right)
    \label{eqn:quantizationloss}
\end{equation}

To 
faithfully learn the quantized classifier $q$, we adopt the following \emph{distillation} loss 
during training, where 
${\bar{y}}$ is the predicted class (by $g$) for input 
$x$ and 
$\mathrm{CE}$ correspond to cross-entropy loss: 

\begin{equation}
    \mathcal{L}_{\textrm{dist}} = \mathcal{L}_{\textrm{quant}} + \mathrm{CE}\left(q(
    \tilde{z})), {\bar{y}} \right)
    \label{eqn:distillation}
\end{equation}

By using 
this loss
,  we strive towards a faithful $q$ to the 
classifier.

\subsection{
Class-specific agents} 
\label{sec:agent}

We obtain $N$ agents, each of which is responsible for arguing wrt a particular class in $\mathcal{Y}$. 
For any given input $x$, an instance of \FAXIC\ is obtained between the two agents whose class is the top-class predicted by  $g(f(x))$  (for the proponent) and the second-best class  (for the opponent).
Specifically, if $g(f(x))=[\bar{y}^1, \dots, \bar{y}^{N}]$,
then the proponent is $\Agents^{\bar{y}^1}$ and
the opponent is $\Agents^{\bar{y}^2}$, as depicted in the assignment block in Figure \ref{fig:overview} (the assignment block selects the agents based on the estimated top two classes).

We model each agent $\Agents^{i}$ (arguing for class $\bar{y}^i$) as a \emph{sequence 
model} $\zeta^{i}$.
In line with \cite{kori2022explaining}, we use gated recurrent units (GRUs) for realizing these sequence models. 
The sequence model $\zeta^{i}$
operates on a hidden state vector ($h^i_{t-1}$)
, which can be treated as a proxy interpretation for 
information 
drawn from the exchange BAF in the \FAXIC\ till the current timestep ($t$).
The sequence model uses a private modulator network $\mathcal{M}^i$
which encodes arguments (in $\ArgsI_{t}$ in the agent's private BAF $\BAFi_t$) to update the hidden state representation:
\begin{align}
    h^i_{t} = \zeta^i(h^i_{t-1}, \mathcal{M}^i(\ArgsI_{t})) 
    \label{eqn:sequence_steps}
\end{align}
The output (hidden state vector) of $\zeta^i$  is then used to determine 
 a \emph{policy function} $\Pi^{i}$, for determining how agents contribute attacks and supports (in the agent's private BAF $\BAFi_t$) to \FAXIC s
, by determining which argument these attacks (if the agent is the opponent $\Agents^j$) or supports (if the agent is the proponent $\Agents^i$) are drawn from:
\begin{align}
    \ArgsI_{t} \sim \Pi^i(h^i_{t-1} \mid 
    \ArgsI_{t-1} \cup \Args^{j}_{t-1}, \bar{y}^i)  
 \end{align}
 Note that these arguments are from $\mathfrak{C}_{\bar{y}^i}$, corresponding to private features in $\mathcal{Z}^{\bar{y}^i}$.
%
The output (hidden state vector) of $\zeta^i$  is also used to obtain the \emph{private classifier} $q^{i}$, as a multi-layer perceptron with hidden state vectors as inputs,
associating them to class confidence in turn used to
assign values to sets of private features as in Section~\ref{sec:EX-IC}.


\paragraph{Training.} 
\label{sec:FAX-impl}
To obtain each agent's sequence model in an end-to-end fashion,
we adapt the REINFORCE learning algorithm from \cite{mnih2014recurrent}.
Specifically, we see 
the next argument prediction/selection as a reinforcement learning task, 
with agents' rewards 
as follows.

\begin{definition}
\label{def:reward}
Let  $\langle \BAFx_0, \ldots, \BAFx_{n}, \Agents_0, \ldots,\Agents_n, \SpeakerM 
    \rangle$ be a \FAXIC\ for explanandum $(x,y)$ amongst agents $\Agents \!=\! \{ \Agents^1, \Agents^2 \}$. 
    Then, for $i\!\in\! \{ 1, 2 \}$ and $t \!\in\! \{1, \ldots, n\}$, 
    $\AgentI$'s \emph{reward} at timestep $t$ is $\rewardI_t = \StanceI(\BAFi_0, \arge)\SFi(
    \BAFi_t, \arge)$
    .\footnote{Here, we treat the agent's initial stance as a (positive/negative) sign  for the explanandum's current strength in the 
     private BAF.}
    \label{def:reward}
\end{definition}

Thus, 
reward is 
a continuous-valued function modelled using the agent's ``confidence'' towards the explanandum, and reflecting the contributed arguments to date and the agent's original stance towards the explanandum. Note that this stance is always positive for the first agent and negative for the second, given our choice of $\tau$ in Section~\ref{sec:EX-IC}. 
Note also that rewards are 
between $[-1, 1]$. 

We combine the agent's reward with REINFORCE gradients
:

\begin{equation}
\mathcal{L}^i_{\mathrm{FAX}} = -\sum_t \log \Pi^i_{\theta^i}(h^i_t \mid  \ArgsI_{t-1}\cup \Args^j_{t-1}, \bar{y}^i) (r^i_t - b^i_t)    
\label{eqn:FAX_loss}
\end{equation}
where  $\theta^i$ are the parameters of policy $\Pi^i$ (and thus modulator $\mathcal{M}^i$) being learnt and $b^i_t$ is a baseline value estimated by $\Agents^i$ at timstep $t$, which is mainly used to reduce the variance in the agent's behaviour during the exploration stage in (reinforcement) learning.
Minimisation of this 
loss can also be viewed as maximisation of log-likelihood of the policy distribution \cite{mnih2014recurrent}.

Finally, to encourage agents to argue for a particular class ($\bar{y}^1$ for the proponent and $\bar{y}^2$ for the opponent) we use
the \emph{stance loss} $\mathcal{L}^i_{\mathrm{stance}} = \mathrm{CE}(\Sigma^i(\BAFi_t,(x,\bar{y}^i)), \bar{y}^i)$ 
to obtain  the combined loss
: $$\mathcal{L}_{\mathrm{total}} = \mathcal{L}_{\mathrm{dist}} +  \Big(\mathcal{L}^1_{\mathrm{FAX}} + \mathcal{L}^1_{\mathrm{stance}}\Big) + \Big(\mathcal{L}^2_{\mathrm{FAX}} + \mathcal{L}^2_{\mathrm{stance}}\Big).$$

\paragraph{
Deployment.}
After training, we deploy the learnt agent
s for generating \FAXIC s
.  
To determine the number of timesteps in each \FAXIC, we adopt the following strategy.
We analyse cosine similarity between arguments to evaluate the information contributed in each timestep 
and we model the gain in information as the average dissimilarity of the contributed argument wrt all the previous arguments.
We terminate the \FAXIC\ if the mean dissimilarity is 
less then a small amount (which is a parameter) or 
when the \FAXIC\ is resolved.

\section{Evaluation}
\label{sec:eval}

In this section we lay out our approach to evaluating the realization of our \FAXIC s 
for explaining image classifiers
.
We 
use evaluation metrics 
for assessing 
\emph{faithfulness} and \emph{argumentative quality} of our \FAXIC s.
The metrics are 
measured  
on
a test set $\mathcal{T}\subseteq \mathcal{X}\times \mathcal{Y}$, providing ground-truths (correct classifications) for a number of inputs (we will use concrete instances of $\mathcal{T}$ in our experiments in Section~\ref{sec:results}).

We define the metrics using the same notation  $\mathfrak{C}(z)$ as in Section~\ref{sec:set-up} 
as well as 
notations $\mathfrak{C}(h)$
to represent the codebook corresponding to
hidden state representation $h$ (in some sequence model)
and $q^i(h)$ to represent the values assigned by $q^i$ 
to the 
private features/arguments corresponding to
$h$ (in $\zeta^i$).

The faithfulness metrics are adapted from the literature, and \cite{kori2022explaining} in particular.
The
first metric measures \emph{correctness} of $q$ and of the codebooks
by measuring accuracy 
wrt the ground-truth in $\mathcal{T}$:
\vspace*{-0.1cm}
        \[
        \mid \{ 
            (x,y) \in \mathcal{T} \mid q(\mathfrak{C}(
            z)) = y 
             \} \mid / \; { \mid \mathcal{T} \mid }
        \]

The second metric measures
\emph{completeness} of $q$ on the BAFs resulting from \FAXIC s, by measuring the 
accuracy of $q$ 
on
the codebooks corresponding to the hidden state representations of the arguments 
in these BAFs,  
wrt the classifier’s predictions on 
$\mathcal{T}$:\footnote{With an abuse of notation we use $n$ to indicate the length of every \FAXIC\ obtained from datapoints in $\mathcal{T}$, even though different \FAXIC s will typically have different lengths. 
}
\vspace*{-0.1cm}
\[ 
           \Big| \{ (x, y) \in \mathcal{T} \mid q\left( \mathfrak{C} \left(\cup_{i=1}^N h^i_n \right) \right) = (g \circ f)(x) \} \Big|
           \; /\;  {\mid \mathcal{T} \mid}
           \]
This metric gives an indication of the faithfulness to the original classifier of $q$ on the output \FAXIC s (as the loss function used during training  only strives towards faithfulness of $q$ on the input to  \FAXIC s).

The argumentative quality metrics 
are tailored to our (implementation of) \FAXIC s. The third metric measures \emph{consensus} amongst the (two) agents
, in terms of the number of resolved \FAXIC s:

\vspace*{-0.1cm}
\[
\Big| \{(x, y) \in \mathcal{T} \mid q^i (h^i_n) = q^j (h^j_n), 
            j \neq i \}\Big|
            \; / \; {\mid \mathcal{T} \mid}
            \]
The fourth (and final) metric (\emph{pro persuasion rate}) measures  consensus again, but towards the proponent agent:   

\vspace*{-0.1cm}
    \[    \Big| \{ (x, y) \in \mathcal{T} \mid q^1 (h^1_0) = q^j (h^j_n),  j \neq 1 \}\Big|
    \; / \; {\mid \mathcal{T} \mid}
    \]
Given that agents are trained to disagree (see Definition~\ref{def:reward}) both 
argumentative metrics can be seen as estimates of the goodness of the learnt class-specific features: high 
values indicate that information is leaked across different features (arguments).  

\section{Experiments}
\label{sec:results}

\begin{table}
    \centering
\caption{Accuracies of the trained classifiers} 
\vspace*{-0.4cm}
\resizebox{1.0\columnwidth}{!}{
        \begin{tabular}{lccc}\\\toprule  
        & \textsc{Fair}   & \textsc{Biased} & \textsc{Random} \\\midrule
        \textsc{AFHQ-ResNet-18}      & 0.95  &  0.39 & 0.31 \\ 
        \textsc{AFHQ-DenseNet-121}   & 0.96  &  0.47 & 0.32 \\ \midrule
        \textsc{FFHQ-ResNet-18}      & 0.88  &  0.58 &  0.47 \\ 
        \textsc{FFHQ-DenseNet-121}   & 0.92  &  0.61 &  0.48 \\
        \bottomrule
        \end{tabular}
    }
\label{table:classifier_results}
\end{table}

\begin{table*}
\centering
\caption{
Faithfulness properties of all considered methods on the DenseNet121 classifiers.}
\label{table:comparative_results}
\vspace*{-0.2cm}
\resizebox{\textwidth}{!}{

\begin{tabular}{c|ccccc|ccccc} 
\toprule 
\multirow{2}{*}{{\begin{tabular}[c]{@{}c@{}}\textsc{Methods} $\rightarrow$ \\ \textsc{Dataset} $\downarrow$\end{tabular}}} 
& \multicolumn{5}{c}{\textbf{\begin{tabular}[c]{@{}c@{}} \textsc{Correctness}\end{tabular}}} 
& \multicolumn{5}{c}{\textbf{\begin{tabular}[c]{@{}c@{}} \textsc{Completeness} \end{tabular}}} \\ 
\cmidrule{2-11} 

& \textsc{GradCAM} & \textsc{DeepLIFT} & \textsc{DeepSHAP} & \textsc{LIME}  & \textsc{FAX}
& \textsc{GradCAM} & \textsc{DeepLIFT} & \textsc{GradSHAP} & \textsc{LIME} & \textsc{FAX} \\ 

\midrule 

\textbf{FFHQ-Random} 

& $0.55$ 
& $0.58$ 
& $0.54$ 
& $0.50$ 
& $0.53$

& $0.40$ 
& $0.41$ 
& $0.38$ 
& $0.36$ 
& $0.97$ 

\\

\textbf{FFHQ-Biased} 

& $0.35$ 
& $0.37$ 
& $0.39$ 
& $0.50$ 
& $0.91$

& $0.34$ 
& $0.36$ 
& $0.40$ 
& $0.48$ 
& $1.00$

\\ 


\textbf{FFHQ-Fair} 

& $0.77$ 
& $0.77$ 
& $0.73$ 
& $0.58$ 
& $0.96$

& $0.74$ 
& $0.74$ 
& $0.68$ 
& $0.59$ 
& $0.96$
\\ 
\midrule 

\textbf{AFHQ-Random} 

& $0.30$ 
& $0.32$ 
& $0.28$ 
& $0.25$ 
& $0.55$

& $0.35$ 
& $0.36$ 
& $0.33$ 
& $0.31$ 
& $0.72$

\\ 


\textbf{AFHQ-Biased} 

& $0.27$ 
& $0.32$ 
& $0.37$ 
& $0.48$ 
& $0.71$

& $0.35$ 
& $0.38$ 
& $0.40$ 
& $0.47$ 
& $0.91$

\\ 


\textbf{AFHQ-Fair} 

& $0.75$ 
& $0.71$ 
& $0.71$ 
& $0.52$ 
& $0.78$

& $0.70$ 
& $0.66$ 
& $0.67$ 
& $0.54$ 
& 0.99

\\ 

\bottomrule 
\end{tabular}
}
\end{table*}

We analyse \FAXIC s on the high resolution animal and human faces (AFHQ\cite{choi2020stargan}, FFHQ \cite{karras2019style}) datasets, with two well known architectures ResNet-18 \cite{he2016deep} and DenseNet121 \cite{huang2017densely} as image classifiers $g \circ f$.
We 
consider three 
settings: (i) \emph{fair}, where the classifier is trained with correct labels; (ii) \emph{biased}, where the classifier is trained with 
biased labels, obtained by randomly switching the labels for 10\% of the datasets
, and (iii) \emph{random}: where we use randomly initialised weights for the classifiers rather than training them
. 

\begin{figure*}
    \hfill
    \includegraphics[width=.423\textwidth]{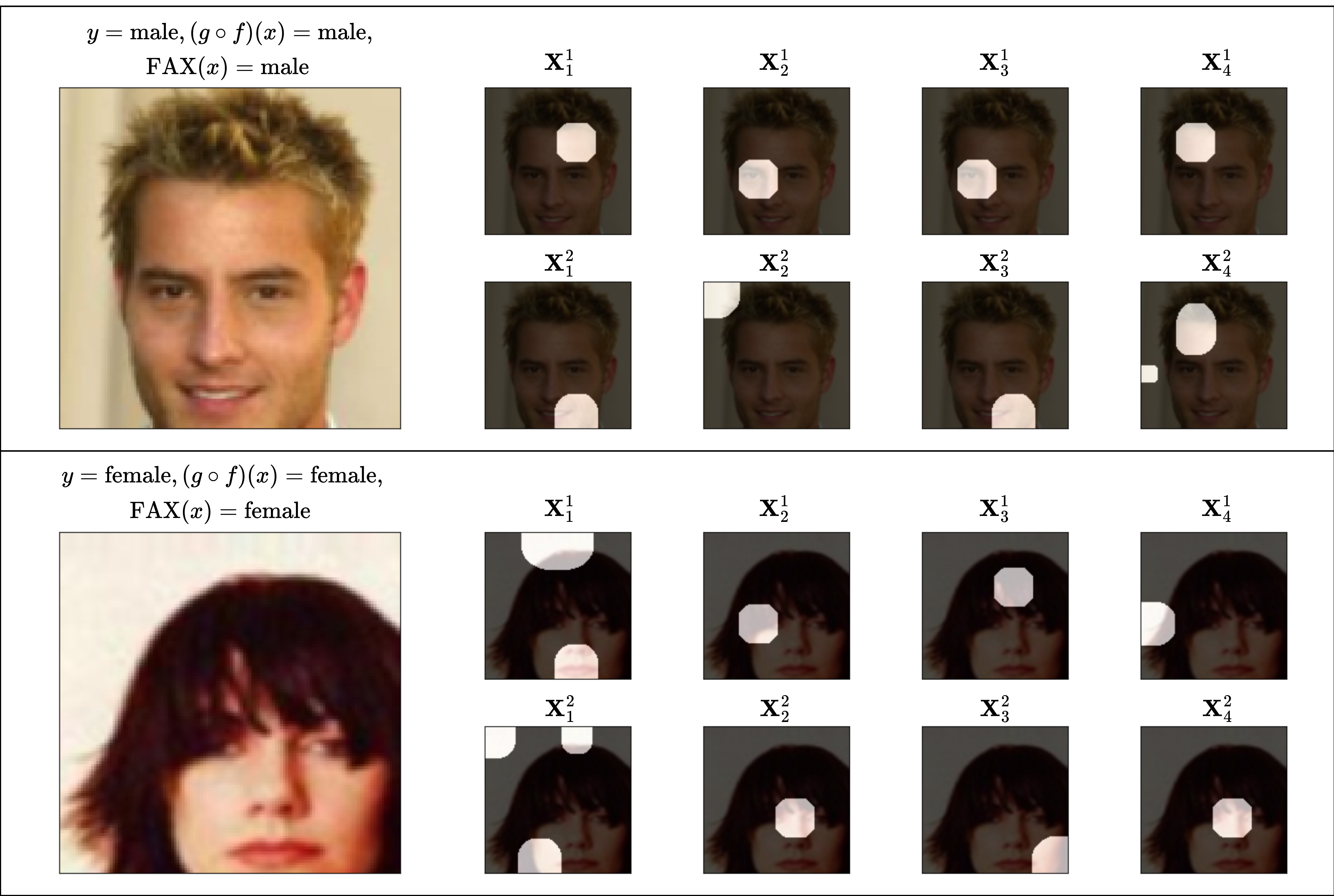} \hfill
    \includegraphics[width=.567\textwidth]{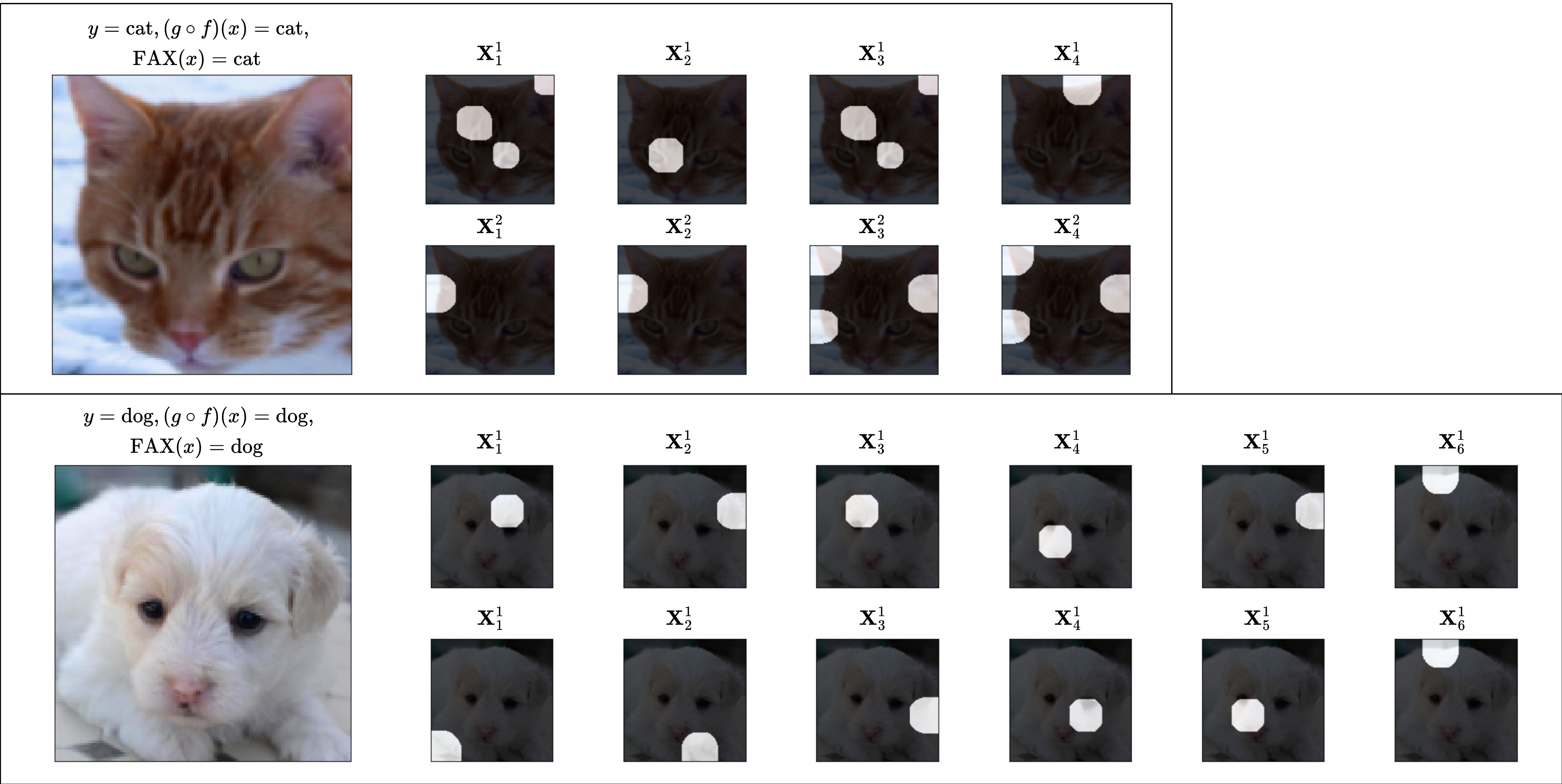} \hfill
    \caption{Arguments 
    in \FAXIC s (
    the proponent starts with the top left argument, the opponent follows with the argument below it, etc.), for classifiers trained on  FFHQ (left column) and AFHQ (right column)
    , on fair (top row) and biased (bottom row) 
    settings. 
    }
    \label{fig:qualitative results}
    \vspace{-5pt}
\end{figure*}

\paragraph{Qualitative results}
Figure \ref{fig:qualitative results} shows some \FAXIC s visualised as in 
\cite{kori2022explaining}, using the approach in \cite{dissection}
.
In all these figures, the first image is the input, while the first and second rows correspond respectively to $\Agents^1$'s and $\Agents^2$'s 
arguments. 
We can see that \FAXIC s 
focus on different but semantically meaningful regions on the input images, 
for both agents.
In the biased setting, as previously described, we expect some leakage of information across different class-specific features/arguments, as observed in 
the figure.

\paragraph{Quantitative results} Table \ref{table:classifier_results} gives
the classifiers' accuracy on test sets.
For the DenseNet121 classifiers,
Table~\ref{table:comparative_results} gives the faithfulness results  
in comparison with standard baseline
s (
\CR{i.e.}
GradCAM \cite{selvaraju2017grad}, DeepLIFT \cite{deeplift},  DeepSHAP \cite{SHAP} and LIME \cite{lime}) and Table~\ref{table:densenet_FAX_results} measures the argumentative metrics in the three settings (for \FAXIC s only, as these metrics are not applicable to baseline methods).
Further,
Table \ref{table:resnet_FAX_results} measure all metrics for the esNet-18 classifiers.
We observe that 
\emph{completeness} is high in all settings, 
while \emph{correctness} is high for the fair and biased setting\CR{s only}
; this is due to completeness reflecting \FAXIC s  behaviour wrt continuous classifier, while correctness is measured wrt ground truth.
As for the last two metrics, the experiments show higher values in the biased than fair settings, as expected, given that  
 we expect the leak of features across codebooks due to incorrect label assignment in the former, which results in higher consensus and easier persuasion.
 We can observe mixed behaviours 
 \CR{with} random classifiers, which results in high value in the case of overlapping features during initialisation and low in the other case.
\CR{Overall, the experiments show that for more ``uncertain'' models (random), the agents cannot reach as much consensus as for ``certain'' models (fair and biased) and the persuasion rate is lower for ``uncertain'' models.
This analysis is empowered by the argumentative nature of our explanatory framework.}

\begin{table}[]
    \centering
    \caption{
    Argumentative metrics for 
    DenseNet-121 classifiers. 
    }
   \vspace*{-0.4cm} \resizebox{1.0\columnwidth}{!}{
        \begin{tabular}{lccc}\\\toprule  
        & \textsc{Fair}   & \textsc{Biased} & \textsc{Random} \\\midrule
        \textbf{AFHQ-Consensus}             &  0.24 &  0.44  &  0.26 \\ 
        \textbf{AFHQ-Pro persuasion rate}   &  0.27 &  0.41  &  0.77 \\
        \midrule
        \textbf{FFHQ-Consensus}             & 0.42  & 0.54   & 0.09 \\ 
        \textbf{FFHQ-Pro persuasion rate}   & 0.33  & 0.50   & 0.38\\
        \bottomrule
        \end{tabular}
    }
\label{table:densenet_FAX_results}
\end{table}

\begin{table}[]
    \centering
    \caption{
    All metrics for ResNet-18 classifiers. 
    }
    \vspace*{-0.4cm}\resizebox{1.0\columnwidth}{!}{
        \begin{tabular}{lccc}\\\toprule  
        & \textsc{Fair}   & \textsc{Biased} & \textsc{Random} \\\midrule
        \textbf{AFHQ-Correctness}            &  0.77 & 0.73 & 0.29 \\
        \textbf{AFHQ-Completeness}           &  0.99 & 0.93 & 0.68 \\
        \textbf{AFHQ-Consensus}             &  0.57 & 0.89 & 0.13 \\ 
        \textbf{AFHQ-Pro persuasion rate}   &  0.48 & 0.56 & 0.16 \\
        \midrule
        \textbf{FFHQ-Correctness}          &  0.65 & 0.45 &  0.51 \\
        \textbf{FFHQ-Completeness}         &  0.99 & 0.99 &  0.91 \\
        \textbf{FFHQ-Consensus}             &  0.31 & 0.90 &  0.45 \\ 
        \textbf{FFHQ-Pro persuasion rate}   &  0.24 & 0.53 &  0.71 \\
        \bottomrule
        \end{tabular}
    }
\label{table:resnet_FAX_results}
\end{table}

\section{
Conclusions}

We have defined explanations for image classification as (free) argumentative exchanges between two 
agents, 
aiming to demystify trained image classifiers based on the argument contribution strategies by 
the agents. 
Differently from standard feature attribution methods generating heatmaps over responsible regions in images, our 
method generates more fine-grained composition of sub-regions, incrementally.
%
Our work 
opens 
many opportunities for future work. 
We plan to investigate whether \FAXIC s can 
uncover shortcuts in classifiers
. 
Further, it would be valuable to collaborate with domain experts to attribute semantic meaning to 
arguments, potentially aiding 
alignment between human understanding and the latent knowledge of models.
Also, it would be 
interesting to apply our approach in settings where quantized representation learning is already explored, ranging from natural images to medical data \cite{van2017neural,esser2021taming,santhirasekaram2022vector}\CR{, or by targeting other (potentially more complex) architectures, e.g. transformers \cite{Vaswani_17}}.
Methodologically,
we also plan to explore the use of object identification methods such as grounded slot attention \cite{ICLR24},
instead of quantization, 
to improve the human understandability of arguments in our \FAXIC s.
\CR{Finally, a promising avenue for fully exploiting the capabilities of FAXs amounts to leveraging notions from \cite{Kenny_23,Santhirasekaram_23,Xie_23} to create hierarchical concepts, giving FAXs with more interesting argument interactions.
}



\begin{acks}
This research was partially supported by the ERC under the EU’s Horizon 2020 research and innovation programme (grant no. 101020934), by J.P. Morgan and the 
RAEng 
(grant no. RCSRF2021/11/45) 
and by the UKRI (grant no. EP/S023356/1) via the CDT in Safe and Trusted Artificial Intelligence.
\end{acks}



\bibliographystyle{ACM-Reference-Format} 
\bibliography{ref}


\begin{thebibliography}{46}


\ifx \showCODEN    \undefined \def \showCODEN     #1{\unskip}     \fi
\ifx \showDOI      \undefined \def \showDOI       #1{#1}\fi
\ifx \showISBNx    \undefined \def \showISBNx     #1{\unskip}     \fi
\ifx \showISBNxiii \undefined \def \showISBNxiii  #1{\unskip}     \fi
\ifx \showISSN     \undefined \def \showISSN      #1{\unskip}     \fi
\ifx \showLCCN     \undefined \def \showLCCN      #1{\unskip}     \fi
\ifx \shownote     \undefined \def \shownote      #1{#1}          \fi
\ifx \showarticletitle \undefined \def \showarticletitle #1{#1}   \fi
\ifx \showURL      \undefined \def \showURL       {\relax}        \fi
\providecommand\bibfield[2]{#2}
\providecommand\bibinfo[2]{#2}
\providecommand\natexlab[1]{#1}
\providecommand\showeprint[2][]{arXiv:#2}

\bibitem[\protect\citeauthoryear{Albini, Lertvittayakumjorn, Rago, and
  Toni}{Albini et~al\mbox{.}}{2020}]%
        {dax}
\bibfield{author}{\bibinfo{person}{Emanuele Albini}, \bibinfo{person}{Piyawat
  Lertvittayakumjorn}, \bibinfo{person}{Antonio Rago}, {and}
  \bibinfo{person}{Francesca Toni}.} \bibinfo{year}{2020}\natexlab{}.
\newblock \showarticletitle{{DAX:} {D}eep {A}rgumentative e{X}planation for
  Neural Networks}.
\newblock \bibinfo{journal}{\emph{CoRR}}  \bibinfo{volume}{abs/2012.05766}
  (\bibinfo{year}{2020}).
\newblock
\urldef\tempurl%
\url{https://arxiv.org/abs/2012.05766}
\showURL{%
\tempurl}


\bibitem[\protect\citeauthoryear{Amgoud and Ben{-}Naim}{Amgoud and
  Ben{-}Naim}{2018}]%
        {Amgoud_18}
\bibfield{author}{\bibinfo{person}{Leila Amgoud} {and}
  \bibinfo{person}{Jonathan Ben{-}Naim}.} \bibinfo{year}{2018}\natexlab{}.
\newblock \showarticletitle{Weighted Bipolar Argumentation Graphs: Axioms and
  Semantics}. In \bibinfo{booktitle}{\emph{IJCAI}}.
  \bibinfo{pages}{5194--5198}.
\newblock
\urldef\tempurl%
\url{https://doi.org/10.24963/IJCAI.2018/720}
\showDOI{\tempurl}


\bibitem[\protect\citeauthoryear{Ayoobi, Kasaei, Cao, Verbrugge, and
  Verheij}{Ayoobi et~al\mbox{.}}{2023}]%
        {hamed}
\bibfield{author}{\bibinfo{person}{Hamed Ayoobi}, \bibinfo{person}{S.~Hamidreza
  Kasaei}, \bibinfo{person}{Ming Cao}, \bibinfo{person}{Rineke Verbrugge},
  {and} \bibinfo{person}{Bart Verheij}.} \bibinfo{year}{2023}\natexlab{}.
\newblock \showarticletitle{Explain What You See: Open-Ended Segmentation and
  Recognition of Occluded {3D} Objects}.
\newblock \bibinfo{journal}{\emph{CoRR}}  \bibinfo{volume}{abs/2301.07037}
  (\bibinfo{year}{2023}).
\newblock
\urldef\tempurl%
\url{https://doi.org/10.48550/arXiv.2301.07037}
\showDOI{\tempurl}


\bibitem[\protect\citeauthoryear{Baroni, Rago, and Toni}{Baroni
  et~al\mbox{.}}{2018}]%
        {Baroni_18}
\bibfield{author}{\bibinfo{person}{Pietro Baroni}, \bibinfo{person}{Antonio
  Rago}, {and} \bibinfo{person}{Francesca Toni}.}
  \bibinfo{year}{2018}\natexlab{}.
\newblock \showarticletitle{How Many Properties Do We Need for Gradual
  Argumentation?}. In \bibinfo{booktitle}{\emph{AAAI}}.
  \bibinfo{pages}{1736--1743}.
\newblock
\urldef\tempurl%
\url{https://doi.org/10.1609/aaai.v32i1.11544}
\showDOI{\tempurl}


\bibitem[\protect\citeauthoryear{Baroni, Rago, and Toni}{Baroni
  et~al\mbox{.}}{2019}]%
        {Baroni_19}
\bibfield{author}{\bibinfo{person}{Pietro Baroni}, \bibinfo{person}{Antonio
  Rago}, {and} \bibinfo{person}{Francesca Toni}.}
  \bibinfo{year}{2019}\natexlab{}.
\newblock \showarticletitle{From fine-grained properties to broad principles
  for gradual argumentation: {A} principled spectrum}.
\newblock \bibinfo{journal}{\emph{Int. J. Approx. Reason.}}
  \bibinfo{volume}{105} (\bibinfo{year}{2019}), \bibinfo{pages}{252--286}.
\newblock
\urldef\tempurl%
\url{https://doi.org/10.1016/J.IJAR.2018.11.019}
\showDOI{\tempurl}


\bibitem[\protect\citeauthoryear{Bau, Zhou, Khosla, Oliva, and Torralba}{Bau
  et~al\mbox{.}}{2017}]%
        {dissection}
\bibfield{author}{\bibinfo{person}{David Bau}, \bibinfo{person}{Bolei Zhou},
  \bibinfo{person}{Aditya Khosla}, \bibinfo{person}{Aude Oliva}, {and}
  \bibinfo{person}{Antonio Torralba}.} \bibinfo{year}{2017}\natexlab{}.
\newblock \showarticletitle{Network Dissection: Quantifying Interpretability of
  Deep Visual Representations}. In \bibinfo{booktitle}{\emph{CVPR}}.
  \bibinfo{pages}{3319--3327}.
\newblock
\urldef\tempurl%
\url{https://doi.org/10.1109/CVPR.2017.354}
\showDOI{\tempurl}


\bibitem[\protect\citeauthoryear{Cawsey}{Cawsey}{1991}]%
        {cawsey1991generating}
\bibfield{author}{\bibinfo{person}{Alison Cawsey}.}
  \bibinfo{year}{1991}\natexlab{}.
\newblock \showarticletitle{Generating Interactive Explanations}. In
  \bibinfo{booktitle}{\emph{AAAI}}. \bibinfo{pages}{86--91}.
\newblock
\urldef\tempurl%
\url{http://www.aaai.org/Library/AAAI/1991/aaai91-014.php}
\showURL{%
\tempurl}


\bibitem[\protect\citeauthoryear{Cayrol and Lagasquie{-}Schiex}{Cayrol and
  Lagasquie{-}Schiex}{2005}]%
        {Cayrol_05}
\bibfield{author}{\bibinfo{person}{Claudette Cayrol} {and}
  \bibinfo{person}{Marie{-}Christine Lagasquie{-}Schiex}.}
  \bibinfo{year}{2005}\natexlab{}.
\newblock \showarticletitle{On the Acceptability of Arguments in Bipolar
  Argumentation Frameworks}. In \bibinfo{booktitle}{\emph{ECSQARU}}.
  \bibinfo{pages}{378--389}.
\newblock
\urldef\tempurl%
\url{https://doi.org/10.1007/11518655\_33}
\showDOI{\tempurl}


\bibitem[\protect\citeauthoryear{Chattopadhyay, Sarkar, Howlader, and
  Balasubramanian}{Chattopadhyay et~al\mbox{.}}{2018}]%
        {gradcampp}
\bibfield{author}{\bibinfo{person}{Aditya Chattopadhyay},
  \bibinfo{person}{Anirban Sarkar}, \bibinfo{person}{Prantik Howlader}, {and}
  \bibinfo{person}{Vineeth~N. Balasubramanian}.}
  \bibinfo{year}{2018}\natexlab{}.
\newblock \showarticletitle{Grad-CAM++: Generalized Gradient-Based Visual
  Explanations for Deep Convolutional Networks}. In
  \bibinfo{booktitle}{\emph{WACV}}. \bibinfo{pages}{839--847}.
\newblock
\urldef\tempurl%
\url{https://doi.org/10.1109/WACV.2018.00097}
\showDOI{\tempurl}


\bibitem[\protect\citeauthoryear{Choi, Uh, Yoo, and Ha}{Choi
  et~al\mbox{.}}{2020}]%
        {choi2020stargan}
\bibfield{author}{\bibinfo{person}{Yunjey Choi}, \bibinfo{person}{Youngjung
  Uh}, \bibinfo{person}{Jaejun Yoo}, {and} \bibinfo{person}{Jung{-}Woo Ha}.}
  \bibinfo{year}{2020}\natexlab{}.
\newblock \showarticletitle{StarGAN v2: Diverse Image Synthesis for Multiple
  Domains}. In \bibinfo{booktitle}{\emph{CVPR}}. \bibinfo{pages}{8185--8194}.
\newblock
\urldef\tempurl%
\url{https://doi.org/10.1109/CVPR42600.2020.00821}
\showDOI{\tempurl}


\bibitem[\protect\citeauthoryear{Cyras, Rago, Albini, Baroni, and Toni}{Cyras
  et~al\mbox{.}}{2021}]%
        {vcyras2021argumentative}
\bibfield{author}{\bibinfo{person}{Kristijonas Cyras}, \bibinfo{person}{Antonio
  Rago}, \bibinfo{person}{Emanuele Albini}, \bibinfo{person}{Pietro Baroni},
  {and} \bibinfo{person}{Francesca Toni}.} \bibinfo{year}{2021}\natexlab{}.
\newblock \showarticletitle{Argumentative {XAI:} {A} Survey}. In
  \bibinfo{booktitle}{\emph{IJCAI}}. \bibinfo{pages}{4392--4399}.
\newblock
\urldef\tempurl%
\url{https://doi.org/10.24963/IJCAI.2021/600}
\showDOI{\tempurl}


\bibitem[\protect\citeauthoryear{de~Tarl{\'{e}}, Bonzon, and
  Maudet}{de~Tarl{\'{e}} et~al\mbox{.}}{2022}]%
        {Tarle_22}
\bibfield{author}{\bibinfo{person}{Louise~Dupuis de Tarl{\'{e}}},
  \bibinfo{person}{Elise Bonzon}, {and} \bibinfo{person}{Nicolas Maudet}.}
  \bibinfo{year}{2022}\natexlab{}.
\newblock \showarticletitle{Multiagent Dynamics of Gradual Argumentation
  Semantics}. In \bibinfo{booktitle}{\emph{AAMAS}}. \bibinfo{pages}{363--371}.
\newblock
\urldef\tempurl%
\url{https://doi.org/10.5555/3535850.3535892}
\showDOI{\tempurl}


\bibitem[\protect\citeauthoryear{Dung}{Dung}{1995}]%
        {dung1995acceptability}
\bibfield{author}{\bibinfo{person}{Phan~Minh Dung}.}
  \bibinfo{year}{1995}\natexlab{}.
\newblock \showarticletitle{On the Acceptability of Arguments and its
  Fundamental Role in Nonmonotonic Reasoning, Logic Programming and n-Person
  Games}.
\newblock \bibinfo{journal}{\emph{Artif. Intell.}} \bibinfo{volume}{77},
  \bibinfo{number}{2} (\bibinfo{year}{1995}), \bibinfo{pages}{321--358}.
\newblock
\urldef\tempurl%
\url{https://doi.org/10.1016/0004-3702(94)00041-X}
\showDOI{\tempurl}


\bibitem[\protect\citeauthoryear{Esser, Rombach, and Ommer}{Esser
  et~al\mbox{.}}{2021}]%
        {esser2021taming}
\bibfield{author}{\bibinfo{person}{Patrick Esser}, \bibinfo{person}{Robin
  Rombach}, {and} \bibinfo{person}{Bj{\"{o}}rn Ommer}.}
  \bibinfo{year}{2021}\natexlab{}.
\newblock \showarticletitle{Taming Transformers for High-Resolution Image
  Synthesis}. In \bibinfo{booktitle}{\emph{CVPR}}.
  \bibinfo{pages}{12873--12883}.
\newblock
\urldef\tempurl%
\url{https://doi.org/10.1109/CVPR46437.2021.01268}
\showDOI{\tempurl}


\bibitem[\protect\citeauthoryear{Goyal, Wu, Ernst, Batra, Parikh, and
  Lee}{Goyal et~al\mbox{.}}{2019}]%
        {goyal2019counterfactual}
\bibfield{author}{\bibinfo{person}{Yash Goyal}, \bibinfo{person}{Ziyan Wu},
  \bibinfo{person}{Jan Ernst}, \bibinfo{person}{Dhruv Batra},
  \bibinfo{person}{Devi Parikh}, {and} \bibinfo{person}{Stefan Lee}.}
  \bibinfo{year}{2019}\natexlab{}.
\newblock \showarticletitle{Counterfactual Visual Explanations}. In
  \bibinfo{booktitle}{\emph{ICML}}. \bibinfo{pages}{2376--2384}.
\newblock
\urldef\tempurl%
\url{http://proceedings.mlr.press/v97/goyal19a.html}
\showURL{%
\tempurl}


\bibitem[\protect\citeauthoryear{He, Zhang, Ren, and Sun}{He
  et~al\mbox{.}}{2016}]%
        {he2016deep}
\bibfield{author}{\bibinfo{person}{Kaiming He}, \bibinfo{person}{Xiangyu
  Zhang}, \bibinfo{person}{Shaoqing Ren}, {and} \bibinfo{person}{Jian Sun}.}
  \bibinfo{year}{2016}\natexlab{}.
\newblock \showarticletitle{Deep Residual Learning for Image Recognition}. In
  \bibinfo{booktitle}{\emph{CVPR}}. \bibinfo{pages}{770--778}.
\newblock
\urldef\tempurl%
\url{https://doi.org/10.1109/CVPR.2016.90}
\showDOI{\tempurl}


\bibitem[\protect\citeauthoryear{Huang, Liu, van~der Maaten, and
  Weinberger}{Huang et~al\mbox{.}}{2017}]%
        {huang2017densely}
\bibfield{author}{\bibinfo{person}{Gao Huang}, \bibinfo{person}{Zhuang Liu},
  \bibinfo{person}{Laurens van~der Maaten}, {and} \bibinfo{person}{Kilian~Q.
  Weinberger}.} \bibinfo{year}{2017}\natexlab{}.
\newblock \showarticletitle{Densely Connected Convolutional Networks}. In
  \bibinfo{booktitle}{\emph{CVPR}}. \bibinfo{pages}{2261--2269}.
\newblock
\urldef\tempurl%
\url{https://doi.org/10.1109/CVPR.2017.243}
\showDOI{\tempurl}


\bibitem[\protect\citeauthoryear{Irving, Christiano, and Amodei}{Irving
  et~al\mbox{.}}{2018}]%
        {debate}
\bibfield{author}{\bibinfo{person}{Geoffrey Irving}, \bibinfo{person}{Paul~F.
  Christiano}, {and} \bibinfo{person}{Dario Amodei}.}
  \bibinfo{year}{2018}\natexlab{}.
\newblock \showarticletitle{{AI} safety via debate}.
\newblock \bibinfo{journal}{\emph{CoRR}}  \bibinfo{volume}{abs/1805.00899}
  (\bibinfo{year}{2018}).
\newblock
\showeprint[arXiv]{1805.00899}
\urldef\tempurl%
\url{http://arxiv.org/abs/1805.00899}
\showURL{%
\tempurl}


\bibitem[\protect\citeauthoryear{Jang, Gu, and Poole}{Jang
  et~al\mbox{.}}{2017}]%
        {jang2016categorical}
\bibfield{author}{\bibinfo{person}{Eric Jang}, \bibinfo{person}{Shixiang Gu},
  {and} \bibinfo{person}{Ben Poole}.} \bibinfo{year}{2017}\natexlab{}.
\newblock \showarticletitle{Categorical Reparameterization with
  Gumbel-Softmax}. In \bibinfo{booktitle}{\emph{ICLR}}.
\newblock
\urldef\tempurl%
\url{https://openreview.net/forum?id=rkE3y85ee}
\showURL{%
\tempurl}


\bibitem[\protect\citeauthoryear{Karras, Laine, and Aila}{Karras
  et~al\mbox{.}}{2019}]%
        {karras2019style}
\bibfield{author}{\bibinfo{person}{Tero Karras}, \bibinfo{person}{Samuli
  Laine}, {and} \bibinfo{person}{Timo Aila}.} \bibinfo{year}{2019}\natexlab{}.
\newblock \showarticletitle{A Style-Based Generator Architecture for Generative
  Adversarial Networks}. In \bibinfo{booktitle}{\emph{CVPR}}.
  \bibinfo{pages}{4401--4410}.
\newblock
\urldef\tempurl%
\url{https://doi.org/10.1109/CVPR.2019.00453}
\showDOI{\tempurl}


\bibitem[\protect\citeauthoryear{Kenny, Delaney, and Keane}{Kenny
  et~al\mbox{.}}{2023}]%
        {Kenny_23}
\bibfield{author}{\bibinfo{person}{Eoin~M. Kenny}, \bibinfo{person}{Eoin
  Delaney}, {and} \bibinfo{person}{Mark~T. Keane}.}
  \bibinfo{year}{2023}\natexlab{}.
\newblock \showarticletitle{Advancing Post-Hoc Case-Based Explanation with
  Feature Highlighting}. In \bibinfo{booktitle}{\emph{{IJCAI}}}.
  \bibinfo{pages}{427--435}.
\newblock
\urldef\tempurl%
\url{https://doi.org/10.24963/IJCAI.2023/48}
\showDOI{\tempurl}


\bibitem[\protect\citeauthoryear{Kori, Glocker, and Toni}{Kori
  et~al\mbox{.}}{2022}]%
        {glance}
\bibfield{author}{\bibinfo{person}{Avinash Kori}, \bibinfo{person}{Ben
  Glocker}, {and} \bibinfo{person}{Francesca Toni}.}
  \bibinfo{year}{2022}\natexlab{}.
\newblock \showarticletitle{{GLANCE:} Global to Local Architecture-Neutral
  Concept-based Explanations}.
\newblock \bibinfo{journal}{\emph{CoRR}}  \bibinfo{volume}{abs/2207.01917}
  (\bibinfo{year}{2022}).
\newblock
\urldef\tempurl%
\url{https://doi.org/10.48550/ARXIV.2207.01917}
\showDOI{\tempurl}
\showeprint[arXiv]{2207.01917}


\bibitem[\protect\citeauthoryear{Kori, Glocker, and Toni}{Kori
  et~al\mbox{.}}{2024a}]%
        {kori2022explaining}
\bibfield{author}{\bibinfo{person}{Avinash Kori}, \bibinfo{person}{Ben
  Glocker}, {and} \bibinfo{person}{Francesca Toni}.}
  \bibinfo{year}{2024}\natexlab{a}.
\newblock \showarticletitle{Explaining Image Classifiers with Visual Debates}.
  In \bibinfo{booktitle}{\emph{{DS}}}. \bibinfo{pages}{200--214}.
\newblock
\urldef\tempurl%
\url{https://doi.org/10.1007/978-3-031-78980-9\_13}
\showDOI{\tempurl}


\bibitem[\protect\citeauthoryear{Kori, Locatello, Ribeiro, Toni, and
  Glocker}{Kori et~al\mbox{.}}{2024b}]%
        {ICLR24}
\bibfield{author}{\bibinfo{person}{Avinash Kori}, \bibinfo{person}{Francesco
  Locatello}, \bibinfo{person}{Fabio De~Sousa Ribeiro},
  \bibinfo{person}{Francesca Toni}, {and} \bibinfo{person}{Ben Glocker}.}
  \bibinfo{year}{2024}\natexlab{b}.
\newblock \showarticletitle{Grounded Object-Centric Learning}. In
  \bibinfo{booktitle}{\emph{ICLR}}.
\newblock
\urldef\tempurl%
\url{https://openreview.net/forum?id=pBxeZ6pVUD}
\showURL{%
\tempurl}


\bibitem[\protect\citeauthoryear{Lakkaraju, Slack, Chen, Tan, and
  Singh}{Lakkaraju et~al\mbox{.}}{2022}]%
        {lakkaraju2022rethinking}
\bibfield{author}{\bibinfo{person}{Himabindu Lakkaraju}, \bibinfo{person}{Dylan
  Slack}, \bibinfo{person}{Yuxin Chen}, \bibinfo{person}{Chenhao Tan}, {and}
  \bibinfo{person}{Sameer Singh}.} \bibinfo{year}{2022}\natexlab{}.
\newblock \showarticletitle{Rethinking Explainability as a Dialogue: {A}
  Practitioner's Perspective}.
\newblock \bibinfo{journal}{\emph{CoRR}}  \bibinfo{volume}{abs/2202.01875}
  (\bibinfo{year}{2022}).
\newblock
\showeprint[arXiv]{2202.01875}
\urldef\tempurl%
\url{https://arxiv.org/abs/2202.01875}
\showURL{%
\tempurl}


\bibitem[\protect\citeauthoryear{Lundberg and Lee}{Lundberg and Lee}{2017}]%
        {SHAP}
\bibfield{author}{\bibinfo{person}{Scott~M. Lundberg} {and}
  \bibinfo{person}{Su{-}In Lee}.} \bibinfo{year}{2017}\natexlab{}.
\newblock \showarticletitle{A Unified Approach to Interpreting Model
  Predictions}. In \bibinfo{booktitle}{\emph{NIPS}}.
  \bibinfo{pages}{4765--4774}.
\newblock
\urldef\tempurl%
\url{https://proceedings.neurips.cc/paper/2017/hash/8a20a8621978632d76c43dfd28b67767-Abstract.html}
\showURL{%
\tempurl}


\bibitem[\protect\citeauthoryear{Madumal, Miller, Sonenberg, and
  Vetere}{Madumal et~al\mbox{.}}{2019}]%
        {madumal2019grounded}
\bibfield{author}{\bibinfo{person}{Prashan Madumal}, \bibinfo{person}{Tim
  Miller}, \bibinfo{person}{Liz Sonenberg}, {and} \bibinfo{person}{Frank
  Vetere}.} \bibinfo{year}{2019}\natexlab{}.
\newblock \showarticletitle{A Grounded Interaction Protocol for Explainable
  Artificial Intelligence}. In \bibinfo{booktitle}{\emph{AAMAS}}.
  \bibinfo{pages}{1033--1041}.
\newblock
\urldef\tempurl%
\url{http://dl.acm.org/citation.cfm?id=3331801}
\showURL{%
\tempurl}


\bibitem[\protect\citeauthoryear{Miller}{Miller}{2019}]%
        {miller2019explanation}
\bibfield{author}{\bibinfo{person}{Tim Miller}.}
  \bibinfo{year}{2019}\natexlab{}.
\newblock \showarticletitle{Explanation in artificial intelligence: Insights
  from the social sciences}.
\newblock \bibinfo{journal}{\emph{Artif. Intell.}}  \bibinfo{volume}{267}
  (\bibinfo{year}{2019}), \bibinfo{pages}{1--38}.
\newblock
\urldef\tempurl%
\url{https://doi.org/10.1016/J.ARTINT.2018.07.007}
\showDOI{\tempurl}


\bibitem[\protect\citeauthoryear{Mnih, Heess, Graves, and Kavukcuoglu}{Mnih
  et~al\mbox{.}}{2014}]%
        {mnih2014recurrent}
\bibfield{author}{\bibinfo{person}{Volodymyr Mnih}, \bibinfo{person}{Nicolas
  Heess}, \bibinfo{person}{Alex Graves}, {and} \bibinfo{person}{Koray
  Kavukcuoglu}.} \bibinfo{year}{2014}\natexlab{}.
\newblock \showarticletitle{Recurrent Models of Visual Attention}. In
  \bibinfo{booktitle}{\emph{NIPS}}. \bibinfo{pages}{2204--2212}.
\newblock
\urldef\tempurl%
\url{https://proceedings.neurips.cc/paper/2014/hash/09c6c3783b4a70054da74f2538ed47c6-Abstract.html}
\showURL{%
\tempurl}


\bibitem[\protect\citeauthoryear{Potyka}{Potyka}{2021}]%
        {Potyka_21}
\bibfield{author}{\bibinfo{person}{Nico Potyka}.}
  \bibinfo{year}{2021}\natexlab{}.
\newblock \showarticletitle{Interpreting Neural Networks as Quantitative
  Argumentation Frameworks}. In \bibinfo{booktitle}{\emph{AAAI}}.
  \bibinfo{pages}{6463--6470}.
\newblock
\urldef\tempurl%
\url{https://doi.org/10.1609/AAAI.V35I7.16801}
\showDOI{\tempurl}


\bibitem[\protect\citeauthoryear{Rago, Baroni, and Toni}{Rago
  et~al\mbox{.}}{2022}]%
        {Rago_22}
\bibfield{author}{\bibinfo{person}{Antonio Rago}, \bibinfo{person}{Pietro
  Baroni}, {and} \bibinfo{person}{Francesca Toni}.}
  \bibinfo{year}{2022}\natexlab{}.
\newblock \showarticletitle{Explaining Causal Models with Argumentation: the
  Case of Bi-variate Reinforcement}. In \bibinfo{booktitle}{\emph{KR}}.
\newblock
\urldef\tempurl%
\url{https://proceedings.kr.org/2022/52/}
\showURL{%
\tempurl}


\bibitem[\protect\citeauthoryear{Rago, Li, and Toni}{Rago
  et~al\mbox{.}}{2023}]%
        {Rago_23}
\bibfield{author}{\bibinfo{person}{Antonio Rago}, \bibinfo{person}{Hengzhi Li},
  {and} \bibinfo{person}{Francesca Toni}.} \bibinfo{year}{2023}\natexlab{}.
\newblock \showarticletitle{Interactive Explanations by Conflict Resolution via
  Argumentative Exchanges}. In \bibinfo{booktitle}{\emph{KR}}.
  \bibinfo{pages}{582--592}.
\newblock
\urldef\tempurl%
\url{https://doi.org/10.24963/kr.2023/57}
\showDOI{\tempurl}


\bibitem[\protect\citeauthoryear{Ribeiro, Singh, and Guestrin}{Ribeiro
  et~al\mbox{.}}{2016}]%
        {lime}
\bibfield{author}{\bibinfo{person}{Marco~T{\'{u}}lio Ribeiro},
  \bibinfo{person}{Sameer Singh}, {and} \bibinfo{person}{Carlos Guestrin}.}
  \bibinfo{year}{2016}\natexlab{}.
\newblock \showarticletitle{"Why Should {I} Trust You?": Explaining the
  Predictions of Any Classifier}. In \bibinfo{booktitle}{\emph{SIGKDD}}.
  \bibinfo{pages}{1135--1144}.
\newblock
\urldef\tempurl%
\url{https://doi.org/10.1145/2939672.2939778}
\showDOI{\tempurl}


\bibitem[\protect\citeauthoryear{Santhirasekaram, Kori, Winkler, Rockall, and
  Glocker}{Santhirasekaram et~al\mbox{.}}{2022}]%
        {santhirasekaram2022vector}
\bibfield{author}{\bibinfo{person}{Ainkaran Santhirasekaram},
  \bibinfo{person}{Avinash Kori}, \bibinfo{person}{Mathias Winkler},
  \bibinfo{person}{Andrea~G. Rockall}, {and} \bibinfo{person}{Ben Glocker}.}
  \bibinfo{year}{2022}\natexlab{}.
\newblock \showarticletitle{Vector Quantisation for Robust Segmentation}. In
  \bibinfo{booktitle}{\emph{MICCAI}}. \bibinfo{pages}{663--672}.
\newblock
\urldef\tempurl%
\url{https://doi.org/10.1007/978-3-031-16440-8\_63}
\showDOI{\tempurl}


\bibitem[\protect\citeauthoryear{Santhirasekaram, Kori, Winkler, Rockall, Toni,
  and Glocker}{Santhirasekaram et~al\mbox{.}}{2023}]%
        {Santhirasekaram_23}
\bibfield{author}{\bibinfo{person}{Ainkaran Santhirasekaram},
  \bibinfo{person}{Avinash Kori}, \bibinfo{person}{Mathias Winkler},
  \bibinfo{person}{Andrea~G. Rockall}, \bibinfo{person}{Francesca Toni}, {and}
  \bibinfo{person}{Ben Glocker}.} \bibinfo{year}{2023}\natexlab{}.
\newblock \showarticletitle{Robust Hierarchical Symbolic Explanations in
  Hyperbolic Space for Image Classification}. In
  \bibinfo{booktitle}{\emph{{CVPR}}}. \bibinfo{pages}{561--570}.
\newblock
\urldef\tempurl%
\url{https://doi.org/10.1109/CVPRW59228.2023.00063}
\showDOI{\tempurl}


\bibitem[\protect\citeauthoryear{Sattarzadeh, Sudhakar, Plataniotis, Jang,
  Jeong, and Kim}{Sattarzadeh et~al\mbox{.}}{2021}]%
        {integratedcam}
\bibfield{author}{\bibinfo{person}{Sam Sattarzadeh}, \bibinfo{person}{Mahesh
  Sudhakar}, \bibinfo{person}{Konstantinos~N. Plataniotis},
  \bibinfo{person}{Jongseong Jang}, \bibinfo{person}{Yeonjeong Jeong}, {and}
  \bibinfo{person}{Hyunwoo Kim}.} \bibinfo{year}{2021}\natexlab{}.
\newblock \showarticletitle{Integrated Grad-Cam: Sensitivity-Aware Visual
  Explanation of Deep Convolutional Networks Via Integrated Gradient-Based
  Scoring}. In \bibinfo{booktitle}{\emph{ICASSP}}. \bibinfo{pages}{1775--1779}.
\newblock
\urldef\tempurl%
\url{https://doi.org/10.1109/ICASSP39728.2021.9415064}
\showDOI{\tempurl}


\bibitem[\protect\citeauthoryear{Selvaraju, Cogswell, Das, Vedantam, Parikh,
  and Batra}{Selvaraju et~al\mbox{.}}{2017}]%
        {selvaraju2017grad}
\bibfield{author}{\bibinfo{person}{Ramprasaath~R. Selvaraju},
  \bibinfo{person}{Michael Cogswell}, \bibinfo{person}{Abhishek Das},
  \bibinfo{person}{Ramakrishna Vedantam}, \bibinfo{person}{Devi Parikh}, {and}
  \bibinfo{person}{Dhruv Batra}.} \bibinfo{year}{2017}\natexlab{}.
\newblock \showarticletitle{Grad-CAM: Visual Explanations from Deep Networks
  via Gradient-Based Localization}. In \bibinfo{booktitle}{\emph{ICCV}}.
  \bibinfo{pages}{618--626}.
\newblock
\urldef\tempurl%
\url{https://doi.org/10.1109/ICCV.2017.74}
\showDOI{\tempurl}


\bibitem[\protect\citeauthoryear{Shitole, Li, Kahng, Tadepalli, and
  Fern}{Shitole et~al\mbox{.}}{2021}]%
        {shitole2021one}
\bibfield{author}{\bibinfo{person}{Vivswan Shitole}, \bibinfo{person}{Fuxin
  Li}, \bibinfo{person}{Minsuk Kahng}, \bibinfo{person}{Prasad Tadepalli},
  {and} \bibinfo{person}{Alan Fern}.} \bibinfo{year}{2021}\natexlab{}.
\newblock \showarticletitle{One Explanation is Not Enough: Structured Attention
  Graphs for Image Classification}. In \bibinfo{booktitle}{\emph{NeurIPS}}.
  \bibinfo{pages}{11352--11363}.
\newblock
\urldef\tempurl%
\url{https://proceedings.neurips.cc/paper/2021/hash/5e751896e527c862bf67251a474b3819-Abstract.html}
\showURL{%
\tempurl}


\bibitem[\protect\citeauthoryear{Shrikumar, Greenside, and Kundaje}{Shrikumar
  et~al\mbox{.}}{2017}]%
        {deeplift}
\bibfield{author}{\bibinfo{person}{Avanti Shrikumar}, \bibinfo{person}{Peyton
  Greenside}, {and} \bibinfo{person}{Anshul Kundaje}.}
  \bibinfo{year}{2017}\natexlab{}.
\newblock \showarticletitle{Learning Important Features Through Propagating
  Activation Differences}. In \bibinfo{booktitle}{\emph{ICML}}.
  \bibinfo{pages}{3145--3153}.
\newblock
\urldef\tempurl%
\url{http://proceedings.mlr.press/v70/shrikumar17a.html}
\showURL{%
\tempurl}


\bibitem[\protect\citeauthoryear{Sukpanichnant, Rago, Lertvittayakumjorn, and
  Toni}{Sukpanichnant et~al\mbox{.}}{2021}]%
        {Purin21}
\bibfield{author}{\bibinfo{person}{Purin Sukpanichnant},
  \bibinfo{person}{Antonio Rago}, \bibinfo{person}{Piyawat Lertvittayakumjorn},
  {and} \bibinfo{person}{Francesca Toni}.} \bibinfo{year}{2021}\natexlab{}.
\newblock \showarticletitle{Neural QBAFs: Explaining Neural Networks Under
  LRP-Based Argumentation Frameworks}. In \bibinfo{booktitle}{\emph{AIxIA}}.
  \bibinfo{pages}{429--444}.
\newblock
\urldef\tempurl%
\url{https://doi.org/10.1007/978-3-031-08421-8\_30}
\showDOI{\tempurl}


\bibitem[\protect\citeauthoryear{Thauvin, Herbin, Ouerdane, and
  Hudelot}{Thauvin et~al\mbox{.}}{2024}]%
        {Thauvin_24}
\bibfield{author}{\bibinfo{person}{Dao Thauvin},
  \bibinfo{person}{St{\'{e}}phane Herbin}, \bibinfo{person}{Wassila Ouerdane},
  {and} \bibinfo{person}{C{\'{e}}line Hudelot}.}
  \bibinfo{year}{2024}\natexlab{}.
\newblock \showarticletitle{Interpretable Image Classification Through an
  Argumentative Dialog Between Encoders}. In
  \bibinfo{booktitle}{\emph{{ECAI}}}. \bibinfo{pages}{3316--3323}.
\newblock
\urldef\tempurl%
\url{https://doi.org/10.3233/FAIA240880}
\showDOI{\tempurl}


\bibitem[\protect\citeauthoryear{van~den Oord, Vinyals, and
  Kavukcuoglu}{van~den Oord et~al\mbox{.}}{2017}]%
        {van2017neural}
\bibfield{author}{\bibinfo{person}{A{\"{a}}ron van~den Oord},
  \bibinfo{person}{Oriol Vinyals}, {and} \bibinfo{person}{Koray Kavukcuoglu}.}
  \bibinfo{year}{2017}\natexlab{}.
\newblock \showarticletitle{Neural Discrete Representation Learning}. In
  \bibinfo{booktitle}{\emph{NIPS}}. \bibinfo{pages}{6306--6315}.
\newblock
\urldef\tempurl%
\url{https://proceedings.neurips.cc/paper/2017/hash/7a98af17e63a0ac09ce2e96d03992fbc-Abstract.html}
\showURL{%
\tempurl}


\bibitem[\protect\citeauthoryear{Vaswani, Shazeer, Parmar, Uszkoreit, Jones,
  Gomez, Kaiser, and Polosukhin}{Vaswani et~al\mbox{.}}{2017}]%
        {Vaswani_17}
\bibfield{author}{\bibinfo{person}{Ashish Vaswani}, \bibinfo{person}{Noam
  Shazeer}, \bibinfo{person}{Niki Parmar}, \bibinfo{person}{Jakob Uszkoreit},
  \bibinfo{person}{Llion Jones}, \bibinfo{person}{Aidan~N. Gomez},
  \bibinfo{person}{Lukasz Kaiser}, {and} \bibinfo{person}{Illia Polosukhin}.}
  \bibinfo{year}{2017}\natexlab{}.
\newblock \showarticletitle{Attention is All you Need}. In
  \bibinfo{booktitle}{\emph{{NeurIPS}}}. \bibinfo{pages}{5998--6008}.
\newblock
\urldef\tempurl%
\url{https://proceedings.neurips.cc/paper/2017/hash/3f5ee243547dee91fbd053c1c4a845aa-Abstract.html}
\showURL{%
\tempurl}


\bibitem[\protect\citeauthoryear{Wang and Vasconcelos}{Wang and
  Vasconcelos}{2019}]%
        {wang2019deliberative}
\bibfield{author}{\bibinfo{person}{Pei Wang} {and} \bibinfo{person}{Nuno
  Vasconcelos}.} \bibinfo{year}{2019}\natexlab{}.
\newblock \showarticletitle{Deliberative Explanations: visualizing network
  insecurities}. In \bibinfo{booktitle}{\emph{NeurIPS}}.
  \bibinfo{pages}{1372--1383}.
\newblock
\urldef\tempurl%
\url{https://proceedings.neurips.cc/paper/2019/hash/68053af2923e00204c3ca7c6a3150cf7-Abstract.html}
\showURL{%
\tempurl}


\bibitem[\protect\citeauthoryear{Xie, Li, Lin, Poon, Cao, and Zhang}{Xie
  et~al\mbox{.}}{2023}]%
        {Xie_23}
\bibfield{author}{\bibinfo{person}{Weiyan Xie}, \bibinfo{person}{Xiao{-}Hui
  Li}, \bibinfo{person}{Zhi Lin}, \bibinfo{person}{Leonard K.~M. Poon},
  \bibinfo{person}{Caleb~Chen Cao}, {and} \bibinfo{person}{Nevin~L. Zhang}.}
  \bibinfo{year}{2023}\natexlab{}.
\newblock \showarticletitle{Two-stage holistic and contrastive explanation of
  image classification}. In \bibinfo{booktitle}{\emph{{UAI}}}.
  \bibinfo{pages}{2335--2345}.
\newblock
\urldef\tempurl%
\url{https://proceedings.mlr.press/v216/xie23a.html}
\showURL{%
\tempurl}


\bibitem[\protect\citeauthoryear{Yoon, Jordon, and van~der Schaar}{Yoon
  et~al\mbox{.}}{2019}]%
        {yoon2018invase}
\bibfield{author}{\bibinfo{person}{Jinsung Yoon}, \bibinfo{person}{James
  Jordon}, {and} \bibinfo{person}{Mihaela van~der Schaar}.}
  \bibinfo{year}{2019}\natexlab{}.
\newblock \showarticletitle{{INVASE:} Instance-wise Variable Selection using
  Neural Networks}. In \bibinfo{booktitle}{\emph{ICLR}}.
\newblock
\urldef\tempurl%
\url{https://openreview.net/forum?id=BJg\_roAcK7}
\showURL{%
\tempurl}


\end{thebibliography}


\end{document}